\documentclass{article}
\usepackage{microtype}
\usepackage{graphicx}
\usepackage{subfigure}
\usepackage{booktabs} %
\usepackage[dvipsnames]{xcolor}
\usepackage[T1]{fontenc}

\usepackage{microtype}      %
\usepackage{pifont}%

\usepackage[backref=page]{hyperref}       %
\usepackage{url}            %
\usepackage{amsfonts}       %
\usepackage{amsmath}
\usepackage{amsthm}
\usepackage{amssymb}
\usepackage{thmtools,thm-restate} %
\usepackage{nicefrac}       %

\usepackage{svg}
\usepackage{wrapfig}
\usepackage{graphicx}
\usepackage{caption}
\usepackage{subcaption}
\usepackage{float}
\usepackage{tablefootnote}
\usepackage{booktabs}       %
\usepackage{multirow}

\let\classAND\AND
\let\AND\relax
\usepackage{algorithmic}

\let\AND\classAND
\AtBeginEnvironment{algorithmic}{\let\AND\algoAND}
\usepackage{algorithm}

\usepackage{awesomebox}
\usepackage{alertmessage}
\usepackage{enumitem}
\usepackage{comment}

\usepackage{xcolor}         %
\usepackage{natbib}         %

\usepackage[accepted]{icml2024}

\usepackage{amsmath}
\usepackage{amssymb}
\usepackage{mathtools}
\usepackage{amsthm}
\usepackage{thmtools, thm-restate}

\theoremstyle{plain}

\theoremstyle{definition}
\newtheorem{definition}{Definition}[section]
\newtheorem{assumption}{Assumption}[section]
\newtheorem{conjecture}{Conjecture}[section]
\newtheorem{researchquestion}{Question}[subsection]
\theoremstyle{remark}

\usepackage[textsize=tiny]{todonotes}

\usepackage{url}

\usepackage{breakurl}

\usepackage{makecell}
\usepackage[capitalize,noabbrev,nameinlink]{cleveref}
\usepackage{crossreftools}

\crefname{section}{\S}{\S}
\crefname{subsection}{\S}{\S}
\crefname{subsubsection}{\S}{\S}
\crefname{figure}{Fig.}{Figs.}
\crefname{prop}{Prop.}{Props.}
\crefname{proposition}{Prop.}{Props.}
\crefname{appendix}{Appx.}{Appxs.}
\crefname{algorithm}{Alg.}{Algs.}
\crefname{theorem}{Thm.}{Thms.}
\crefname{conjecture}{Conj.}{Conjs.}
\crefname{researchquestion}{Q.}{Qs.}
\crefname{definition}{Defn.}{Defns.}
\crefname{cor}{Cor.}{Cors.}
\crefname{lem}{Lem.}{Lems.}
\crefname{table}{Tab.}{Tabs.}
\crefname{assum}{Assum.}{Assums.}
\crefname{example}{Ex.}{Exs.}

\newcommand{\ie}{i.e.\@\xspace}

\newcommand{\eg}{e.g.\@\xspace}

\newcommand{\cf}{cf.\@\xspace} %

\newcommand{\wrt}{w.r.t.\@\xspace} %
\newcommand{\inv}[1]{\ensuremath{#1^{-1}}}

\newcommand{\mat}[1]{\ensuremath{\boldsymbol{\mathrm{#1}}}}

\newcommand{\abs}[1]{\ensuremath{\left|#1\right|}}

\newcommand{\marginal}[1]{\ensuremath{p\parenthesis{#1}}}

\newcommand{\expnum}[2]{\ensuremath{{#1}\mathrm{e}{#2}}}

\newcommand{\expectation}[1]{\ensuremath{\mathbb{E}_{#1}}}
\newcommand{\rr}[1]{\ensuremath{\mathbb{R}^{#1}}}

\newcommand{\parenthesis}[1]{\ensuremath{\left(#1\right)}}
\newcommand{\brackets}[1]{\ensuremath{\left[#1\right]}}
\newcommand{\braces}[1]{\ensuremath{\left\{#1\right\}}}

\newcommand{\kl}[2]{\ensuremath{\text{KL}\brackets{#1|| #2}}}

 \renewcommand{\algorithmiccomment}[1]{//#1}
\newcommand{\patrik}[1]{\textcolor{cyan}{[\textbf{Patrik:} #1]}}

\usepackage{amsmath}
\usepackage{nomencl}
\usepackage{xcolor}
\DeclareMathOperator*{\argmax}{argmax}

\newcommand{\epsnonident}{$\varepsilon-$non-identifiability\xspace}

\usepackage[most]{tcolorbox}
\usepackage{tikz}

\newcommand{\pdfset}{\ensuremath{\mathcal{M}}\xspace}
\newcommand{\domain}{\ensuremath{\mathcal{X}}\xspace}
\newcommand{\property}{\ensuremath{\mathcal{A}}\xspace}

\newcommand{\pdfclass}{\ensuremath{\mathcal{P}}\xspace}
\newcommand{\funcclass}{\ensuremath{\mathcal{F}}\xspace}
\newcommand{\biasclass}{\ensuremath{\mathcal{B}}\xspace}

\usepackage{marginnote}

\usepackage[acronym, automake, toc, nomain, nopostdot, style=tree, nonumberlist,numberedsection]{glossaries}
\usepackage{glossary-mcols}
\usepackage{xparse}
\usepackage{xspace}
\setglossarystyle{mcolindex}

\newglossary{abbrev}{abs}{abo}{Nomenclature}

\newglossaryentry{aux}{
    name        = \ensuremath{\mathrm{\boldsymbol{u}}} ,
    description = {auxiliary variable} ,
    type        = abbrev,
}

\newglossaryentry{im}{
    name        = \ensuremath{\mathrm{Im}} ,
    description = {image space} ,
    type        = abbrev,
}

\newglossaryentry{ker}{
    name        = \ensuremath{\mathrm{Ker}} ,
    description = {kernel space} ,
    type        = abbrev,
}

\newglossaryentry{kronecker}{
    name        = \ensuremath{\otimes} ,
    description = {Kronecker product} ,
    type        = abbrev,
}

\newglossaryentry{loss}{
    name        = \ensuremath{\mathcal{L}} ,
    description = {loss function} ,
    type        = abbrev,
}

\newglossaryentry{numenv}{
    name        = \ensuremath{\abs{E}} ,
    description = {number of environments} ,
    type        = abbrev,
}

\newglossaryentry{lr}{
    name        = \ensuremath{\eta} ,
    description = {learning rate} ,
    type        = abbrev,
}

\newglossaryentry{hypersphere}{
    name        = \ensuremath{\mathcal{S}} ,
    description = {hypersphere} ,
    type        = abbrev,
}

\newglossaryentry{dec}{
    name        = \ensuremath{\boldsymbol{f}} ,
    description = {decoder map $\gls{Latent}\to\gls{Obs}$} ,
    type        = abbrev,
}

\newglossaryentry{deccomp}{
    name        = \ensuremath{f} ,
    description = {decoder map component} ,
    type        = abbrev,
}

\newglossaryentry{enc}{
    name        = \ensuremath{\boldsymbol{g}} ,
    description = {encoder map $\gls{Obs}\to\gls{Latent}$} ,
    type        = abbrev,
}

\newglossaryentry{numdata}{
    name        = \ensuremath{n} ,
    description = {number of samples} ,
    type        = abbrev,
}

\newglossaryentry{observations}{type=abbrev,name=Observations,description={\nopostdesc}}

\newglossaryentry{obs}{
    name        = \ensuremath{\boldsymbol{x}} ,
    description = {observation vector} ,
    type        = abbrev,
    parent      = observations,
}

\newglossaryentry{obscomp}{
    name        = \ensuremath{x} ,
    description = {observation single component} ,
    type        = abbrev,
    parent      = observations,
}

\newglossaryentry{Obs}{
    name        = \ensuremath{\mathcal{X}} ,
    description = {observation space} ,
    type        = abbrev,
    parent      = observations,
}

\newglossaryentry{obsdim}{
    name        = \ensuremath{D} ,
    description = {dimensionality of the observation space \gls{Obs}} ,
    type        = abbrev,
    parent      = observations,
}

\newglossaryentry{obsmat}{
    name        = \ensuremath{\mat{X}} ,
    description = {observation matrix of \rr{\gls{numdata}\times\gls{obsdim}}} ,
    type        = abbrev,
    parent      = observations,
}

\newglossaryentry{obspos}{
    name        = \ensuremath{\tilde{\boldsymbol{x}}} ,
    description = {positive observation vector} ,
    type        = abbrev,
    parent      = observations,
}

\newglossaryentry{obsneg}{
    name        = \ensuremath{{\boldsymbol{x}}^{-}} ,
    description = {negative observation vector} ,
    type        = abbrev,
    parent      = observations,
}

\newglossaryentry{labels}{type=abbrev,name=Labels,description={\nopostdesc}}

\newglossaryentry{label}{
    name        = \ensuremath{\boldsymbol{y}} ,
    description = {label vector} ,
    type        = abbrev,
    parent      = labels,
}

\newglossaryentry{labelhat}{
    name        = \ensuremath{\widehat{\boldsymbol{y}}} ,
    description = {estimated label vector} ,
    type        = abbrev,
    parent      = labels,
}

\newglossaryentry{labelcomp}{
    name        = \ensuremath{y} ,
    description = {label component} ,
    type        = abbrev,
    parent      = labels,
}

\newglossaryentry{labelcomphat}{
    name        = \ensuremath{\widehat{y}} ,
    description = {label component} ,
    type        = abbrev,
    parent      = labels,
}

\newglossaryentry{labelset}{
    name        = \ensuremath{\mathcal{Y}} ,
    description = {label set} ,
    type        = abbrev,
    parent      = labels,
}

\newglossaryentry{labeldim}{
    name        = \ensuremath{C} ,
    description = {number of classes in the label set \gls{labelset}} ,
    type        = abbrev,
    parent      = labels,
}

\newglossaryentry{latents}{type=abbrev,name=Latents,description={\nopostdesc}}

\newglossaryentry{latent}{
    name        = \ensuremath{\boldsymbol{z}} ,
    description = {latent vector} ,
    type        = abbrev,
    parent     = latents,
}

\newglossaryentry{latentcomp}{
    name        = \ensuremath{z} ,
    description = {latent single component} ,
    type        = abbrev,
    parent     = latents,
}

\newglossaryentry{Latent}{
    name        = \ensuremath{\mathcal{Z}} ,
    description = {latents} ,
    type        = abbrev,
    parent     = latents,
}

\newglossaryentry{latentdim}{
    name        = \ensuremath{d} ,
    description = {dimensionality of the latent space \gls{Latent}} ,
    type        = abbrev,
    parent     = latents,
}

\newglossaryentry{latentmat}{
    name        = \ensuremath{\mat{Z}} ,
    description = {latent matrix of \rr{\gls{numdata}\times\gls{latentdim}}} ,
    type        = abbrev,
    parent      = latents,
}

\newglossaryentry{latentpos}{
    name        = \ensuremath{\tilde{\boldsymbol{z}}} ,
    description = {positive latent vector} ,
    type        = abbrev,
    parent      = latents,
}

\newglossaryentry{latentneg}{
    name        = \ensuremath{\boldsymbol{z}^{-}} ,
    description = {negative latent vector} ,
    type        = abbrev,
    parent      = observations,
}

\newglossaryentry{sigmaz}{
    name        = \ensuremath{\boldsymbol{\sigma}_{\gls{latentcomp}}} ,
    description = {std of \gls{latentcomp}} ,
    type        = abbrev,
    parent     = latents,
}

\newglossaryentry{content}{
    name        = \ensuremath{\boldsymbol{z}^{c}} ,
    description = {content latent vector} ,
    type        = abbrev,
    parent     = latents,
}

\newglossaryentry{contentcomp}{
    name        = \ensuremath{z^{c}} ,
    description = {content latent single component} ,
    type        = abbrev,
    parent     = latents,
}

\newglossaryentry{Content}{
    name        = \ensuremath{\mathcal{Z}^{c}} ,
    description = {content} ,
    type        = abbrev,
    parent     = latents,
}

\newglossaryentry{contentdim}{
    name        = \ensuremath{d_{c}} ,
    description = {dimensionality of \gls{content}} ,
    type        = abbrev,
    parent     = latents,
}

\newglossaryentry{sigmac}{
    name        = \ensuremath{\boldsymbol{\sigma}_{c}} ,
    description = {std of \gls{contentcomp}} ,
    type        = abbrev,
    parent     = latents,
}

\newglossaryentry{style}{
    name        = \ensuremath{\boldsymbol{z}^{s}} ,
    description = {style latent vector} ,
    type        = abbrev,
    parent     = latents,
}

\newglossaryentry{stylecomp}{
    name        = \ensuremath{z^{s}} ,
    description = {style latent single component} ,
    type        = abbrev,
    parent     = latents,
}

\newglossaryentry{Style}{
    name        = \ensuremath{\mathcal{Z}^{s}} ,
    description = {style} ,
    type        = abbrev,
    parent     = latents,
}

\newglossaryentry{styledim}{
    name        = \ensuremath{d_{s}} ,
    description = {dimensionality of \gls{style}} ,
    type        = abbrev,
    parent     = latents,
}

\newglossaryentry{sigmas}{
    name        = \ensuremath{\boldsymbol{\sigma}_{s}} ,
    description = {std of \gls{stylecomp}} ,
    type        = abbrev,
    parent     = latents,
}

\newglossaryentry{modality}{
    name        = \ensuremath{\boldsymbol{z}^{m}} ,
    description = {modality-specific latent vector} ,
    type        = abbrev,
    parent     = latents,
}

\newglossaryentry{modalitycomp}{
    name        = \ensuremath{z^{m}} ,
    description = {modality-specific  latent single component} ,
    type        = abbrev,
    parent     = latents,
}

\newglossaryentry{Modality}{
    name        = \ensuremath{\mathcal{Z}^{m}} ,
    description = {latent subspace of \gls{modality}} ,
    type        = abbrev,
    parent     = latents,
}

\newglossaryentry{modalitydim}{
    name        = \ensuremath{d_{m}} ,
    description = {dimensionality of \gls{modality}} ,
    type        = abbrev,
    parent     = latents,
}

\newglossaryentry{algebra}{type=abbrev,name=Algebra,description={\nopostdesc}}

\newglossaryentry{identity}{
    name        = \ensuremath{\boldsymbol{\mathrm{I}}} ,
    description = { identity matrix} ,
    type        = abbrev,
    parent      = algebra,
}

\newglossaryentry{ones}{
    name        = \ensuremath{\boldsymbol{\mathrm{1}}} ,
    description = {a vector of ones} ,
    type        = abbrev,
    parent      = algebra,
}

\newglossaryentry{zeros}{
    name        = \ensuremath{\boldsymbol{\mathrm{0}}} ,
    description = {a vector of zeros} ,
    type        = abbrev,
    parent      = algebra,
}

\newglossaryentry{jacobian}{
    name        = \ensuremath{\boldsymbol{\mathrm{J}}} ,
    description = {Jacobian matrix} ,
    type        = abbrev,
    parent      = algebra,
}

\newglossaryentry{hessian}{
    name        = \ensuremath{\boldsymbol{\mathrm{H}}} ,
    description = {Hessian matrix} ,
    type        = abbrev,
    parent      = algebra,
}

\newglossaryentry{d}{
    name        = \ensuremath{\boldsymbol{\mathrm{D}}} ,
    description = {diagonal matrix} ,
    type        = abbrev,
    parent      = algebra,
}

\newglossaryentry{o}{
    name        = \ensuremath{\boldsymbol{\mathrm{O}}},
    description = {orthogonal matrix} ,
    type        = abbrev,
    parent      = algebra,
}

\newglossaryentry{scalar}{
    name        = \ensuremath{\alpha} ,
    description = {scalar field} ,
    type        = abbrev,
    parent      = algebra,
}

\newglossaryentry{perm}{
    name        = \ensuremath{\mathbb{P}} ,
    description = {group of permutation matrices} ,
    type        = abbrev,
    parent      = algebra,
}

\newglossaryentry{p}{
    name        = \ensuremath{\mat{P}},
    description = {permutation matrix} ,
    type        = abbrev,
    parent      = algebra,
}

\newglossaryentry{prob}{type=abbrev,name=Probability theory,description={\nopostdesc}}

\newglossaryentry{cov}{
    name        = \ensuremath{\boldsymbol{\mathrm{\Sigma}}},
    description = {covariance matrix} ,
    type        = abbrev,
    parent      = prob,
}

\newglossaryentry{mean}{
    name        = \ensuremath{\boldsymbol{\mu}},
    description = {mean} ,
    type        = abbrev,
    parent      = prob,
}

\newglossaryentry{std}{
    name        = \ensuremath{\boldsymbol{\sigma}},
    description = {standard deviation} ,
    type        = abbrev,
    parent      = prob,
}

\newglossaryentry{entropy}{
    name        = \ensuremath{\mathrm{H}} ,
    description = {entropy} ,
    type        = abbrev,
    parent      = prob,
}

\newglossaryentry{expfamparam}{
    name        = \ensuremath{\boldsymbol{\theta}} ,
    description = {parameter of exponential family} ,
    type        = abbrev,
    parent      = prob,
}

\newglossaryentry{expfamnatparam}{
    name        = \ensuremath{\boldsymbol{\eta}} ,
    description = {natural parameter of exponential family} ,
    type        = abbrev,
    parent      = prob,
}

\newglossaryentry{expfamsuffstat}{
    name        = \ensuremath{T(\gls{obs})} ,
    description = {sufficient statistics of exponential family} ,
    type        = abbrev,
    parent      = prob,
}

\newglossaryentry{expfamlogpartition}{
    name        = \ensuremath{A} ,
    description = {log parition function of exponential family (depends on \gls{expfamnatparam})} ,
    type        = abbrev,
    parent      = prob,
}

\newglossaryentry{wishart}{
    name        = \ensuremath{\mathcal{W}} ,
    description = {Wishart distribution} ,
    type        = abbrev,
    parent      = prob,
}

\newglossaryentry{normal}{
    name        = \ensuremath{\mathcal{N}} ,
    description = {normal distribution} ,
    type        = abbrev,
    parent      = prob,
}

\newglossaryentry{matrixnormal}{
    name        = \ensuremath{\mathcal{MN}} ,
    description = {normal distribution} ,
    type        = abbrev,
    parent      = prob,
}

\newglossaryentry{causal}{type=abbrev,name=Causality,description={\nopostdesc}}

\newglossaryentry{cause}{
    name        = \ensuremath{\boldsymbol{N}},
    description = {noise (independent)  variable vector} ,
    type        = abbrev,
    parent      = causal,
}

\newglossaryentry{causecomp}{
    name        = \ensuremath{N},
    description = {noise (independent)  variable component} ,
    type        = abbrev,
    parent      = causal,
}

\newglossaryentry{Cause}{
    name        = \ensuremath{\mathcal{N}} ,
    description = {space of the noise variables} ,
    type        = abbrev,
    parent      = causal,
}

\newglossaryentry{effect}{
    name        = \ensuremath{\boldsymbol{X}},
    description = {observation vector} ,
    type        = abbrev,
    parent      = causal,
}
\newglossaryentry{effectcomp}{
    name        = \ensuremath{X},
    description = {observation component} ,
    type        = abbrev,
    parent      = causal,
}

\newglossaryentry{Effect}{
    name        = \ensuremath{\mathcal{X}} ,
    description = {space of the effect variables} ,
    type        = abbrev,
    parent      = causal,
}

\newglossaryentry{pa}{
    name        = \ensuremath{\boldsymbol{Pa}},
    description = {parents of \gls{effect}} ,
    type        = abbrev,
    parent      = causal,
}

\newglossaryentry{nondesc}{
    name        = \ensuremath{\boldsymbol{ND}},
    description = {non-descendants of \gls{effect}} ,
    type        = abbrev,
    parent      = causal,
}

\newglossaryentry{nondescminuspa}{
    name        = \ensuremath{\boldsymbol{\overline{ND}}},
    description = {non-descendants of \gls{effect}, excluding its parents} ,
    type        = abbrev,
    parent      = causal,
}

\newglossaryentry{semf}{
    name        = \ensuremath{\boldsymbol{f}},
    description = {structural assignment in \glspl{sem}} ,
    type        = abbrev,
    parent      = causal,
}

\newglossaryentry{semfcomp}{
    name        = \ensuremath{f},
    description = {a component of \gls{semf}} ,
    type        = abbrev,
    parent      = causal,
}

\newglossaryentry{order}{
    name        = \ensuremath{\pi},
    description = {causal ordering} ,
    type        = abbrev,
    parent      = causal,
}

\newglossaryentry{indexset}{
    name        = \ensuremath{\mathcal{I}},
    description = {index set} ,
    type        = abbrev,
    parent      = causal,
}
\newglossaryentry{adjacency}{
    name        = \ensuremath{\boldsymbol{\mathcal{A}}} ,
    description = {adjacency matrix of a \glspl{sem}} ,
    type        = abbrev,
    parent      = causal,
}

\newglossaryentry{connectivity}{
    name        = \ensuremath{\boldsymbol{\mathcal{C}}} ,
    description = {connectivity matrix of a \glspl{sem}} ,
    type        = abbrev,
    parent      = causal,
}

\newglossaryentry{dependency}{
    name        = \ensuremath{\mathcal{D}} ,
    description = {dependency matrix of a \glspl{sem}} ,
    type        = abbrev,
    parent      = causal,
}

\newglossaryentry{seq}{
    name        = \ensuremath{\sim_{\acrshort{dag}}} ,
    description = {structural equivalence} ,
    type        = abbrev,
    parent      = causal,
}

\newglossaryentry{contrastive}{type=abbrev,name=Contrastive Learning,description={\nopostdesc}}
\newglossaryentry{clloss}{
    name        = \ensuremath{\mathcal{L}_{\mathrm{\acrshort{cl}}}} ,
    description = {contrastive loss function} ,
    type        = abbrev,
    parent      = contrastive,
}

\newglossaryentry{alignloss}{
    name        = \ensuremath{\mathcal{L}_{\mathrm{align}}} ,
    description = {alignment term in \gls{clloss}} ,
    type        = abbrev,
    parent      = contrastive,
}

\newglossaryentry{uniformloss}{
    name        = \ensuremath{\mathcal{L}_{\mathrm{uniform}}} ,
    description = {uniformity term in \gls{clloss}} ,
    type        = abbrev,
    parent      = contrastive,
}

\newglossaryentry{temp}{
    name        = \ensuremath{{\boldsymbol{\tau}}} ,
    description = {temperature in \gls{clloss}} ,
    type        = abbrev,
    parent      = contrastive,
}

\newglossaryentry{numneg}{
    name        = \ensuremath{M} ,
    description = {number of negative samples} ,
    type        = abbrev,
    parent      = contrastive,
}

\newglossaryentry{vaes}{type=abbrev,name=\acrlongpl{vae},description={\nopostdesc}}

\newglossaryentry{q}{
    name        = \ensuremath{q_{\gls{encpar}}(\gls{latent}|\gls{obs})} ,
    description = {variational posterior of the \acrshort{vae}, mapping $\gls{obs}\mapsto\gls{latent}$ parametrized by \gls{encpar}} ,
    type        = abbrev,
    parent      = vaes,
}
\newglossaryentry{qopt}{
    name        = \ensuremath{q_{\widehat{\gls{encpar}}}(\gls{latent}|\gls{obs})} ,
    description = {optimal variational posterior of the \acrshort{vae}, mapping $\gls{obs}\mapsto\gls{latent}$ parametrized by \gls{encpar}} ,
    type        = abbrev,
    parent      = vaes,
}

\newglossaryentry{encpar}{
    name        = \ensuremath{\boldsymbol{\phi}} ,
    description = {parameters of the variational posterior \gls{q}} ,
    type        = abbrev,
    parent      = vaes,
}

\newglossaryentry{encparopt}{
    name        = \ensuremath{\widehat{\boldsymbol{\phi}}} ,
    description = {optimal parameters of the variational posterior \gls{q}} ,
    type        = abbrev,
    parent      = vaes,
}

\newglossaryentry{var_family}{
    name        = \ensuremath{\mathcal{Q}} ,
    description = {distribution family of the variational posterior \gls{q} } ,
    type        = abbrev,
    parent      = vaes,
}

\newglossaryentry{pz}{
    name        = \ensuremath{p_0(\gls{latent})} ,
    description = {latent prior distribution} ,
    type        = abbrev,
    parent      = vaes,
}

\newglossaryentry{px}{
    name        = \ensuremath{p_{\gls{decpar}}(\gls{obs})} ,
    description = {marginal likelihood } ,
    type        = abbrev,
    parent      = vaes,
}

\newglossaryentry{pdata}{
    name        = \ensuremath{p(\gls{obs})} ,
    description = {data distribution } ,
    type        = abbrev,
    parent      = vaes,
}

\newglossaryentry{mean_enc}{
    name        = \ensuremath{\mu_{\gls{latent}|\gls{obs}}} ,
    description = {mean encoder of the \acrshort{vae}, \ie, $\expectation{\gls{latent}\sim\gls{q}}\parenthesis{\gls{latent}}$, mapping $\gls{obs}\mapsto\gls{latent}$} ,
    type        = abbrev,
    parent      = vaes,
}

\newglossaryentry{var_cov}{
    name        = \ensuremath{\gls{cov}^{\gls{encpar}}_{\gls{latent}|\gls{obs}}} ,
    description = {covariance matrix of \gls{q}} ,
    type        = abbrev,
    parent      = vaes,
}

\newglossaryentry{sigmak}{
    name        = \ensuremath{{\sigma}_{k}^{\gls{encpar}}(\gls{obs})^{2}} ,
    description = {variance of \gls{q} in dimension $k$} ,
    type        = abbrev,
    parent      = vaes,
}

\newglossaryentry{sigmaopt}{
    name        = \ensuremath{\boldsymbol{\sigma}^{\gls{encparopt}}(\gls{obs})^{2}} ,
    description = {optimal variance of \gls{q}} ,
    type        = abbrev,
    parent      = vaes,
}

\newglossaryentry{sigmaoptk}{
    name        = \ensuremath{{\sigma}_{k}^{\gls{encparopt}}(\gls{obs})^{2}} ,
    description = {optimal variance of \gls{q} in dimension $k$} ,
    type        = abbrev,
    parent      = vaes,
}

\newglossaryentry{mu}{
    name        = \ensuremath{\boldsymbol{\mu}^{\gls{encpar}}(\gls{obs})} ,
    description = {mean of \gls{q}} ,
    type        = abbrev,
    parent      = vaes,
}

\newglossaryentry{muk}{
    name        = \ensuremath{{\mu}_{k}^{\gls{encpar}}(\gls{obs})} ,
    description = {mean of \gls{q} in dimension $k$} ,
    type        = abbrev,
    parent      = vaes,
}

\newglossaryentry{muopt}{
    name        = \ensuremath{\boldsymbol{\mu}^{\gls{encparopt}}(\gls{obs})} ,
    description = {optimal mean of \gls{q}} ,
    type        = abbrev,
    parent      = vaes,
}

\newglossaryentry{muoptk}{
    name        = \ensuremath{{\mu}_{k}^{\gls{encparopt}}(\gls{obs})} ,
    description = {optimal mean of \gls{q} in dimension $k$} ,
    type        = abbrev,
    parent      = vaes,
}

\newglossaryentry{gamma}{
    name        = \ensuremath{\gamma} ,
    description = {square root of the precision of the \gls{vae} decoder} ,
    type        = abbrev,
    parent      = vaes,
}

\newglossaryentry{betaloss}{
    name        = \ensuremath{\mathcal{L}_{\beta}} ,
    description = {\betavae loss function} ,
    type        = abbrev,
    parent      = vaes,
}

\newglossaryentry{pxz}{
    name        = \ensuremath{p_{\gls{decpar}}(\gls{obs}|\gls{latent})} ,
    description = {conditional distribution of the decoded samples of the \acrshort{vae}, mapping $\gls{latent}\mapsto\gls{obs}$, parametrized by \gls{decpar}} ,
    type        = abbrev,
    parent      = vaes,
}
\newglossaryentry{pzx}{
    name        = \ensuremath{p_{\gls{decpar}}(\gls{latent}|\gls{obs})} ,
    description = {true posterior distribution of the decoded samples of the \acrshort{vae}, mapping $\gls{obs}\mapsto\gls{latent}$, parametrized by \gls{decpar}} ,
    type        = abbrev,
    parent      = vaes,
}
\newglossaryentry{decpar}{
    name        = \ensuremath{\boldsymbol{\theta}} ,
    description = {parameters of the decoder \gls{pxz}} ,
    type        = abbrev,
    parent      = vaes,
}

\newglossaryentry{invdeccomp}{
    name        = \ensuremath{{g}^{\gls{decpar}}} ,
    description = {inverse decoder component} ,
    type        = abbrev,
    parent      = vaes,
}

\newglossaryentry{invdec}{
    name        = \ensuremath{\mathrm{\boldsymbol{g}}^{\gls{decpar}}} ,
    description = {inverse decoder} ,
    type        = abbrev,
    parent      = vaes,
}

\newglossaryentry{distortion}{
    name        = \ensuremath{D} ,
    description = {Distortion of \cite{alemi_fixing_2018}, the same as the reconstruction term of the \acrshort{elbo} for $\beta=1$} ,
    type        = abbrev,
    parent      = vaes,
}

\newglossaryentry{rate}{
    name        = \ensuremath{R} ,
    description = {Rate of \cite{alemi_fixing_2018}, the same as the \acrshort{kld} term of the \acrshort{elbo} for $\beta=1$} ,
    type        = abbrev,
    parent      = vaes,
}

\newglossaryentry{lindec}{
    name        = \ensuremath{\boldsymbol{\mathrm{W}}} ,
    description = {weight matrix of a linear decoder} ,
    type        = abbrev,
    parent      = vaes,
}

\newglossaryentry{linenc}{
    name        = \ensuremath{\boldsymbol{\mathrm{V}}} ,
    description = {weight matrix of a linear encoder} ,
    type        = abbrev,
    parent      = vaes,
}

\newglossaryentry{imas}{type=abbrev,name=\acrlong{ima},description={\nopostdesc}}

\newglossaryentry{mixing}{
    name        = \ensuremath{\inv{g}} ,
    description = {inverse of the learned unmixing of the \acrshort{ima}, mapping $\gls{latent}\mapsto\gls{obs}$ } ,
    type        = abbrev,
    parent      = imas,
}

\newglossaryentry{lin_mixing}{
    name        = \ensuremath{A} ,
    description = {ground-truth \emph{linear} mixing process of the \acrshort{ima}, mapping $\gls{latent}\mapsto\gls{obs}$ } ,
    type        = abbrev,
    parent      = imas,
}

\newglossaryentry{cima_local}{
    name        = \ensuremath{c_{\acrshort{ima}}} ,
    description = {local \acrshort{ima} contrast } ,
    type        = abbrev,
    parent      = imas,
}

\newglossaryentry{cima_global}{
    name        = \ensuremath{C_{\acrshort{ima}}} ,
    description = {global \acrshort{ima} contrast } ,
    type        = abbrev,
    parent      = imas,
}

\newglossaryentry{source}{
    name        = \ensuremath{s} ,
    description = {sources (\acrshort{ica} equivalent of latents)} ,
    type        = abbrev,
    parent      = imas,
}

\newglossaryentry{rec_s}{
    name        = \ensuremath{\boldsymbol{y}} ,
    description = {reconstructed sources} ,
    type        = abbrev,
    parent      = imas,
}

\newglossaryentry{p_source}{
    name        = \ensuremath{p_{\gls{latent}}} ,
    description = {source distribution} ,
    type        = abbrev,
    parent      = imas,
}

\newglossaryentry{imaloss}{
    name        = \ensuremath{\mathcal{L}_{\gls{ima}}} ,
    description = {\gls{ima} loss function} ,
    type        = abbrev,
    parent      = imas,
}

\NewDocumentCommand{\cima}{ O{\gls{dec}} O{\gls{latent}}  }{\ensuremath{\gls{cima_local} ( #1\!,  #2) }\xspace}
\NewDocumentCommand{\Cima}{ O{\gls{dec}} O{\ensuremath{p_0}  }}{\ensuremath{\gls{cima_global} ( #1,  #2) }\xspace}

\newglossaryentry{gps}{type=abbrev,name=\acrlongpl{gp},description={\nopostdesc}}
\newglossaryentry{gpr}{
    name        = \ensuremath{\mathcal{GP}} ,
    description = {Gaussian Process} ,
    type        = abbrev,
    parent      = gps,
}

\newglossaryentry{gpker}{
    name        = \ensuremath{k} ,
    description = {kernel function} ,
    type        = abbrev,
    parent      = gps,
}

\newglossaryentry{gpcov}{
    name        = \ensuremath{\mathcal{K}} ,
    description = {$\gls{numdata}\times\gls{numdata}$ covariance matrix of a \acrshort{gp}} ,
    type        = abbrev,
    parent      = gps,
}

\makeglossaries

\newacronym{mpa}{MPA}{Measure Preserving Automorphism}
\newacronym{iid}{i.i.d.}{independent and identically distributed}
\newacronym{vmf}{vMF}{von Mises-Fisher}
\newacronym{nivmf}{nivMF}{non-isotropic von Mises-Fisher}
\newacronym{pd}{PD}{positive definite}
\newacronym{psd}{PSD}{positive semi-definite}
\newacronym{nd}{ND}{negative definite}
\newacronym{nsd}{NSD}{negative semi-definite}

\newacronym{ode}{ODE}{Ordinary Differential Equation}
\newacronym{pde}{PDE}{Partial Differential Equation}

\newacronym{lhs}{LHS}{left hand side}
\newacronym{rhs}{RHS}{right hand side}

\newacronym{rv}{RV}{random variable}

\newacronym{ae}{AE}{AutoEncoder}
\newacronym{lae}{LAE}{Linear Autoencoder}
\newacronym{vae}{VAE}{Variational Autoencoder}
\newacronym{cvvae}{CV-VAE}{Constant-Variance Variational Autoencoder}
\newacronym{ivae}{iVAE}{Identifiable Variational Autoencoder}
\newacronym{rae}{RAE}{Regularized Autoencoder}
\newacronym{grae}{GRAE}{Gaussian Regularized Autoencoder}
\newacronym{lvm}{LVM}{latent variable model}

\newacronym[longplural=Gaussian Processes]{gp}{GP}{Gaussian Process}
\newacronym{gplvm}{GPLVM}{Gaussian Process Latent Variable Model}
\newacronym{rbf}{RBF}{Radial Basis Function}

\newcommand{\betavae}{$\beta$-\gls{vae}\xspace}

\newacronym{kld}{KL}{Kullback-Leibler Divergence}
\newacronym{elbo}{{\text{\upshape ELBO}}}{evidence lower bound}
\newacronym{pca}{PCA}{Principal Component Analysis}
\newacronym{ppca}{PPCA}{Probabilistic Principal Component Analysis}
\newacronym{ebm}{EBM}{Energy-Based Model}
\newacronym{cca}{CCA}{Canonical Correlation Analysis}

\newacronym{mi}{MI}{Mutual Information}

\newacronym{icm}{ICM}{Independent Causal Mechanisms}
\newacronym{sms}{SMS}{Sparse Mechanism Shift}
\newacronym{sem}{SEM}{Structural Equation Model}
\newacronym{lingam}{LiNGAM}{Linear Non-Gaussian Acyclic Model}
\newacronym{dag}{DAG}{Directed Acyclic Graph}
\newacronym{anm}{ANM}{Additive Noise Model}
\newacronym{cd}{CD}{Causal Discovery}
\newacronym{crl}{CRL}{Causal Representation Learning}
\newacronym{hmm}{HMM}{Hidden Markov Model}

\newacronym{ica}{ICA}{Independent Component Analysis}
\newacronym{nlica}{NLICA}{nonlinear Independent Component Analysis}
\newacronym{bss}{BSS}{Blind Source Separation}

\newacronym{ima}{{\text{\upshape IMA}}}{Independent Mechanism Analysis}
\newacronym{igci}{IGCI}{Information Geometric Causal Inference}
\newacronym{cdf}{CdF}{Causal de Finetti}

\newacronym{nce}{NCE}{Noise Contrastive Estimation}
\newacronym{pcl}{PCL}{Permutation-Contrastive Learning}
\newacronym{tcl}{TCL}{Time-Contrastive Learning}
\newacronym{iia}{IIA}{Independent Innovation Analysis}

\newacronym{ar}{AR}{autoregressive}
\newacronym{var}{VAR}{Vector autoregressive}
\newacronym{nvar}{NVAR}{Nonlinear Vector AutoRegressive}

\newacronym{ai}{AI}{Artificial Intelligence}
\newacronym{ml}{ML}{Machine Learning}
\newacronym{dl}{DL}{Deep Learning}
\newacronym{rl}{RL}{Reinforcement Learning}
\newacronym{rlhf}{RLHF}{Reinforcement Learning from Human Feedback}
\newacronym{ssl}{SSL}{Self-Supervised Learning}

\newacronym{cl}{CL}{Contrastive Learning}
\newacronym{dcl}{DCL}{Debiased Contrastive Learning}
\newacronym{scl}{SCL}{Spectral Contrastive Learning}
\newacronym{gcl}{GCL}{Graph Contrastive Learning}
\newacronym{alphacl}{$\alpha$-CL}{$\alpha$-Contrastive Learning}
\newacronym{arcl}{ArCL}{Augmentation-robust Contrastive Learning}
\newacronym{fce}{FCE}{Flow Contrastive Estimation}

\newacronym{vince}{VINCE}{Variational InfoNCE}
\newacronym{rince}{RINCE}{Robust InfoNCE}
\newacronym{aggnce}{AggNCE}{Aggregated InfoNCE}
\newacronym{mcinfonce}{MCInfoNCE}{Monte-Carlo InfoNCE}

\newacronym{gmc}{GMC}{Geometric Multimodal Contrastive Learning}
\newacronym{looc}{LooC}{Leave-one-out Contrastive Learning}
\newacronym{npc}{NPC}{Negative-Positive Coupling}
\newacronym{cpc}{CPC}{Contrastive Predictive Coding}
\newacronym{nlp}{NLP}{Natural Language Processing}
\newacronym{gdl}{GDL}{Geometric Deep Learning}
\newacronym{msn}{MSN}{Masked Siamese Networks}
\newacronym{ifm}{IFM}{Implicit Feature Modification}

\newacronym{dnn}{DNN}{Deep Neural Network}
\newacronym{nn}{NN}{Neural Network}
\newacronym{ann}{ANN}{Artificial Neural Network}

\newacronym{fm}{FM}{Foundation Model}
\newacronym{llm}{LLM}{Large Language Model}
\newacronym{pcfg}{PCFG}{Probabilistic Context-Free Grammar}
\newacronym{icl}{ICL}{in-context learning}

\newacronym{nc}{NC}{Neural Collapse}
\newacronym{cdt}{CDT}{Class-Dependent Temperature}
\newacronym{mlp}{MLP}{Multi-Layer Perceptron}
\newacronym{fc}{FC}{Fully Connected}

\newacronym{cn}{conv}{Convolutional layer}
\newacronym{cnn}{CNN}{Convolutional Neural Network}
\newacronym{gnn}{GNN}{Graph Neural Network}

\newacronym{rnn}{RNN}{Recurrent Neural Network}
\newacronym{lstm}{LSTM}{Long Short-Term Memory}
\newacronym{gru}{GRU}{Gated Recurrent Unit}
\newacronym{relu}{ReLU}{Rectified Linear Unit}
\newacronym{bn}{BN}{Batch Normalization}
\newacronym{dbn}{DBN}{Decorrelated Batch Normalization}

\newacronym{gan}{GAN}{Generative Adversarial Network}

\newacronym{sgd}{SGD}{Stochastic Gradient Descent}
\newacronym{adam}{ADAM}{Adaptive Moment Estimation}
\newacronym{svd}{SVD}{Singular Value Decomposition}
\newacronym{wls}{WLS}{Weighted Least Squares}

\newacronym{sam}{SAM}{Sharpness-Aware Minimization}
\newacronym{samba}{SAMBA}{SAM-Based Autoencoder}
\newacronym{vi}{VI}{Variational Inference}
\newacronym{mfvi}{MFVI}{Mean Field Variational Inference}

\newacronym{dgp}{DGP}{Data Generating Process}
\newacronym{map}{MAP}{Maximum A Posteriori}
\newacronym{mle}{MLE}{maximum likelihood estimation}

\newacronym{etf}{ETF}{Equiangular Tight Frame}

\newacronym{mse}{MSE}{Mean Squared Error}
\newacronym{mae}{MAE}{Mean Absolute Error}
\newacronym{ce}{{\text{\upshape CE}}}{cross entropy}

\newacronym{sid}{SID}{Structural Intervention Distance}
\newacronym{shd}{SHD}{Structural Hamming Distance}

\newacronym{mcc}{MCC}{Mean Correlation Coefficient}
\newacronym{mig}{MIG}{Mutual Information Gap}
\newacronym{dci}{DCI}{Disentanglement Completeness Informativeness score}

\newacronym{arc}{ARC}{Average Relative Confusion}
\newacronym{acr}{ACR}{Average Confusion Ratio}

\newacronym{api}{API}{Application Programming Interface}
\newacronym{cpu}{CPU}{Central Processing Unit}
\newacronym{gpu}{GPU}{Graphics Processing Unit}

\newacronym{lti}{LTI}{Linear Time-Invariant}
\newacronym{zoh}{ZOH}{Zero-Order Hold}

\newacronym{gt}{{\text{\upshape GT}}}{ground truth}
\newacronym{ood}{OOD}{out-of-distribution}
\newacronym{oov}{OOV}{out-of-variable}
\newacronym{fsm}{FSM}{Finite State Machine}
\newacronym{rasp}{RASP}{Restricted-Access Sequence Processing Language}

\newacronym{ntk}{NTK}{Neural Tangent Kernel}

\newacronym{as}{a.s.}{almost surely}
\newacronym{alev}{a.e.}{almost everywhere}

\newacronym{sos}{SOS}{start-of-sequence}
\newacronym{eos}{EOS}{end-of-sequence}
\definecolor{figblue}{HTML}{4A90E2}
\definecolor{figred}{HTML}{D0021B}
\definecolor{figgreen}{HTML}{2CA02C}

\everypar{\looseness=-1}

\icmltitlerunning{
Position: 
Understanding LLMs Requires More Than Statistical Generalization}

\begin{document}

\setlength{\parskip}{3pt}
\setlength{\textfloatsep}{15pt}

\twocolumn[
\icmltitle{
Position: 
Understanding LLMs Requires More Than Statistical Generalization}

\icmlsetsymbol{equal}{*}
\icmlsetsymbol{senior}{$\dagger$}

\begin{icmlauthorlist}
\icmlauthor{Patrik Reizinger}{equal,tueb}
\icmlauthor{Szilvia Ujváry}{equal,camb,ucl}
\icmlauthor{Anna Mészáros}{equal,camb}
\icmlauthor{Anna Kerekes}{equal,eth,cls}\\
\vspace{1em}
\icmlauthor{Wieland Brendel}{tueb,ellis,tuai}
\icmlauthor{Ferenc Huszár}{senior,camb}

\end{icmlauthorlist}

\icmlaffiliation{tueb}{Max Planck Institute for Intelligent Systems, Tübingen, Germany}
\icmlaffiliation{ellis}{ELLIS Institute Tübingen, Germany}
\icmlaffiliation{tuai}{Tübingen AI Center, Germany}
\icmlaffiliation{camb}{University of Cambridge, UK}
\icmlaffiliation{eth}{Department of Computer Science, ETH Zurich}
\icmlaffiliation{ucl}{AI Center, UCL, London, UK}
\icmlaffiliation{cls}{Max Planck ETH Center for Learning Systems}

\icmlcorrespondingauthor{Patrik Reizinger}{patrik.reizinger@tuebingen.mpg.de}

\icmlkeywords{Machine Learning, ICML}

\vskip 0.3in
]

\printAffiliationsAndNotice{\icmlEqualContribution \textsuperscript{\textdagger} Senior author} %

\begin{abstract}
    The last decade has seen blossoming research in deep learning theory attempting to answer, ``Why does deep learning generalize?" A powerful shift in perspective precipitated this progress: the study of overparametrized models in the interpolation regime.
    In this paper, we argue that another perspective shift is due, since some of the desirable qualities of LLMs are not a consequence of good statistical generalization and require a separate theoretical explanation. Our core argument relies on the observation that AR probabilistic models are inherently non-identifiable: models zero or near-zero KL divergence apart---thus, equivalent test loss---can exhibit markedly different behaviors. We support our position with mathematical examples and empirical observations, illustrating why non-identifiability has practical relevance through three case studies: (1) the non-identifiability of zero-shot rule extrapolation; (2) the approximate non-identifiability of in-context learning; and (3) the non-identifiability of fine-tunability. We review promising research directions focusing on LLM-relevant generalization measures, transferability, and inductive biases.
\end{abstract}

\newcommand{\szilvi}[1]{\textcolor{red}{[\textbf{Szilvi:} #1]}}

\newcommand{\manna}[1]{\textcolor{purple}{[\textbf{Manna:} #1]}}

\newcommand{\anna}[1]{\textcolor{pink}{[\textbf{Anna:} #1]}}

\section{Introduction}
\label{sec:intro}

\begin{table}[ht]
    \centering
    \setlength{\tabcolsep}{1.6pt}
    \begin{tabular}{lcc} \toprule
          &\textbf{Interpolation regime} & \textbf{Saturation regime}\\\midrule
        \multirow{2}{*}{$\min \gls{loss}_{\mathrm{train}}$}   & non-unique& non-unique\\
        &global min.\ reached& global min.\ reached \\\midrule
        \multirow{2}{*}{$\min \gls{loss}_{\mathrm{test}}$}   & \multirow{2}{*}{{\color{gray}no assumption}} & non-unique \\
        & & global min.\ reached \\ \midrule
        Questions & \makecell{Is $\gls{loss}_{\mathrm{test}}$ \\ small enough?} & \makecell{zero-shot extrapolation\\ \acrlong{icl}\\ transfer, finetunability}\\
        \bottomrule
    \end{tabular}
    \caption{\textbf{Comparison of Interpolation and Saturation regimes:} In the Interpolation regime we assume a global minimum of the training loss is found, but is not unique, so we ask whether the minimum we find generalises well. In the Saturation regime, we further assume that a global minimum of the test loss is found, but even that is not unique. We study additional properties of the minimum found, which are not implied by good generalization.}
    \label{table:regimes}
\end{table}

\Acrfull{ar} language models trained on the next-token prediction objective can have remarkable reasoning~\citep{ouyang2022training,touvron2023llama,wei2022chain}, \gls{icl}~\citep{in_context_bayesian,zhang_trained_2023,min2022rethinking}, and data-efficient fine-tuning capabilities~\citep{brown2020language,liu_same_2023}.\\
Modern theory of deep learning studies neural networks in the~\textit{interpolation regime}~\citep{zhang_understanding_2016,masegosa_learning_2020,Kawaguchi_2022}, \ie, when at the end of training, a model reaches a (non-unique) global minimum of the training loss. Here, the question is whether this model will also perform well on the \textit{in-distribution} test set. Since \glspl{llm} are trained on massive datasets, these models achieve both low training and test loss; thus, they generalize in the statistical sense. However, statistical generalization cannot guarantee good performance on downstream tasks~\citep{liu_same_2023}. \\
\textbf{This position paper argues that we ought to study \glspl{llm} in the {\textit{\textbf{saturation regime}}}}\footnote{The term ``saturation regime'' has also been used to refer to saturation of hidden units \citep{glorot_saturation}. Here we mean exclusively the saturation of the test loss.}~\citep[coined by][]{liu_same_2023} instead (\cref{table:regimes}). In the saturation regime, models reach the (non-unique) global minimum of the test loss during training; since the same minimal test loss cannot distinguish between \gls{ood} model performance~\citep{liu_same_2023}, we should ask \textbf{what additional properties hold for the minimum found by our algorithms}. To formalize such questions, we need to substitute the black box concept of average risk from statistical learning theory with more application-specific goals, \eg, rule extrapolation or the data efficiency of fine-tuning.

\textit{Why is the minimum of the test loss non-unique?} We can use the lens of identifiability to understand this non-trivial question. Namely, unless their support spans the entire space of sequences, \gls{ar} probabilistic models are non-identifiable: different models are indistinguishable by likelihood, even in the limit of infinite data. The study of (non-)identifiability has a vast literature both in statistical inference and causal discovery. These are well-known results; we only aim to highlight the practical implications of non-identifiability for \gls{ar} \glspl{llm}.
\begin{figure*}[t]
\vspace{-1em}
    \centering
    \includegraphics[trim={0 16cm 19cm 0},clip,width=0.62\textwidth]{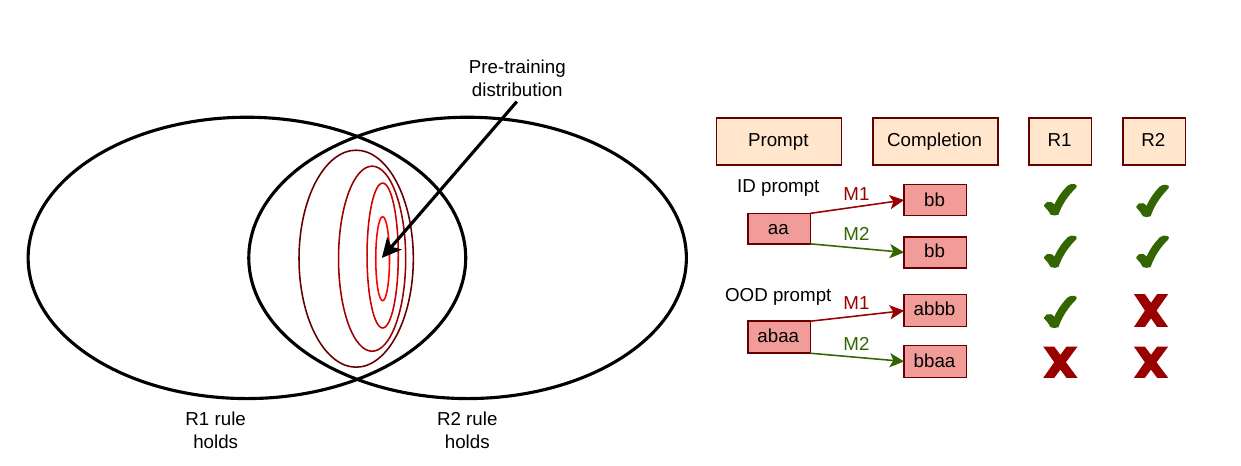}
    \caption{\textbf{Illustration of case study \ref{sec:case-study-1}:} We train a Transformer on a \gls{pcfg} generating sequences of the form $a^nb^n$. \textbf{Left:} This language can be represented as an intersection of two rules: \ref{rule1} the number of $a$s and $b$s match; and \ref{rule2} $a$ never follows a $b$. \textbf{Right:} We consider different models ({\color{red}M1}, {\color{figgreen}M2}) which achieve perfect test loss. On prompts consistent with the $a^nb^n$ grammar (\eg, $aa$) models produce the same completions. However, on prompts that are inconsistent with $a^nb^n$, and thus have probability zero under the pre-training distribution, the models may produce different completions. For these \gls{ood} prompts, we can ask if the completed prompts still satisfy rule \ref{rule1}, which we call rule extrapolation. Rule extrapolation behaviour is not implied by minimal test loss, but may arise due to inductive biases.}
    \label{fig:case-study-1}
\end{figure*}
Our \textbf{contributions} are:
\begin{itemize}[nolistsep,leftmargin=*]
    \item highlighting the limitation of statistical generalization for understanding \gls{ar} \glspl{llm} in the saturation regime (\cref{sec:bg});
    \item demonstrating the relevance of our position through three case studies which provide well-defined starting points to theoretically study \glspl{llm}  (\cref{sec:implications});
    \item summarizing three topics that should be studied in \glspl{llm} in the saturation regime (\cref{sec:position}), and proposing potential research directions (\cref{sec:discussion}).
\end{itemize}

\section{Background}
\label{sec:bg}

\textbf{Statistical generalization} measures whether a model's performance on the training data transfers to unseen test data, assumed to be sampled from the same distribution (\ie, \acrshort{iid}). Classical results in statistical learning theory attempt to bound the generalization gap in terms of uniform notions of the model class' complexity~\citep{Vapnik1971, Vapnik2000, Bartlett2002}. These, however, fail to account for the success of deep learning as argued by \citet{zhang_understanding_2016} and others. More applicable to deep learning are approaches that provide bounds based on the properties of the learning algorithm or the specific hypothesis learned. These include PAC-Bayes ~\citep{dziugaite_computing_2017,perez_ortiz_tighter,lotfi2022so_tight,lotfi2023llm}, information-theoretic~\citep{Russo2016, XuRaginsky2017, Wang2023} and algorithmic stability bounds~\citep{Bousquet02, Deng2021}.

The appeal of focusing on statistical generalization stems from its ``black-box" nature: it can be applied across many domains without changing terminology, be it machine translation or image processing. Stating that an algorithm generalises well requires no domain-specific understanding or qualitative description of any form of generalization behaviour, beyond the specification of a loss function. Slightly more domain-specific thinking is often introduced when one studies \acrfull{ood} generalization~\citep[see][for a review]{Femming2022}, since one needs to describe how the test and training distributions differ.

\textbf{Interpolation regime.}
Overparametrized models gave rise to the \textit{interpolation regime} (\cref{table:regimes}), where a model has enough parameters to (almost) perfectly fit the training data~\citep{zhang_understanding_2016, masegosa2023understanding, Kawaguchi_2022}. In the interpolation regime, the training loss \textit{alone} cannot distinguish whether a model will generalize, yet models that we find by minimizing the training loss typically generalize well. In rare cases, optimization can switch between generalizing and non-generalizing solutions, as observed in grokking \citep{power_grokking_2022}. This observation led to a paradigm shift in the community, inviting researchers to consider training dynamics and the inductive biases enabling statistical generalization instead of only relying on the model class, loss, and dataset structure. In~\cref{sec:position}, we argue that a second paradigm shift is needed: \textit{we should move beyond the interpolation regime and study LLMs in the saturation regime, uncovering the inductive biases responsible for \gls{ood} generalization.}

\textbf{Identifiability of Probabilistic Models.}
Identifiability is an important property of a class of statistical models, determining whether a model can always be uniquely recovered from observed data. In parametric statistical models, identifiability asks whether the parameters of a model are uniquely determined by the data distribution they define~\citep[see \eg ][]{comon1994independent}. In latent variable modeling, we may not care about parameters and ask instead whether the statistical relationship between latent and observed variables can be uniquely determined from the marginal distribution of observables, perhaps up to some invariances~\citep[see, \eg, ][]{hyvarinen_independent_2013}. In causal discovery, one asks whether the causal structure of a distribution (described, \eg, by a specific Markov factorisation) can be uniquely determined from the joint distribution~\citep{hyvarinen_identifiability_2023}.

In machine learning, identifiability can be interpreted as a guarantee that the test loss has a unique minimizer, a unique Bayes optimal model. This is a highly desirable quality as it allows us to reason about properties of this possibly unreachable but unique minimum, for example, predict \gls{ood} extrapolation or the effect of interventions~\citep{pearl_causality_2009}.

\section{Identifiability in \gls{ar} \glspl{llm}}
\label{sec:implications}

In this section, we discuss the identifiability of \gls{ar} probabilistic models. By an \gls{ar} probabilistic model we mean (for some fixed $L \in \mathbb{N}$) a collection $\{p(x_i | x_{1:i-1});L\geq i\geq 1\}$ of conditional distributions, which also define a collection of joint distribution $\{p(x_{1:i}) ; L \geq i\geq1\}$ over sequences. This collection of conditional distributions usually shares a set of parameters $\theta$.
We start by outlining \textit{three important notions of non-identifiability} that one might be interested in when studying such models:

\textbf{(i) Functional non-identifiability}
    happens when the collection of conditionals is not uniquely determined by the collection of joint distributions they define. As we will see, functional identifiability is only guaranteed when all finite prefixes occur with non-zero probability. If this is not the case, two models that differ only in how they complete zero-probability prefixes are indistinguishable by the \gls{kld}. We illustrate implications of this on extrapolation behaviour in \cref{sec:case-study-1};\\
\textbf{(ii) $\varepsilon-$non-identifiablity} 
    is a new term we introduce, which relaxes the notion of functional non-identifiability. It posits that there exist properties of models $p$ that other models $q$ very close to $p$ (at most $\varepsilon$ in \gls{kld}-sense) do not possess. This means that even though a unique global minimum exists, there are near-minima that differ from the global minimum in some important ways. We illustrate that the emergent \gls{icl} ability in some \gls{ar} models is an example of this in \cref{sec:case-study-2};\\
\textbf{(iii) Parameter non-identifiability} 
     means that \textit{functionally equivalent} models can be described by different sets of parameters. While this does not affect the model's zero-shot test performance, it can have significant implications for transfer learning and fine-tuning (\cref{sec:case-study-3}).

\gls{ar} probabilistic models are inherently non-identifiable. Multiple models with perfect generalization may exist and may behave differently. Here we showcase what this means for \gls{ar} \glspl{llm} via three case studies, matching the three notions of non-identifiability from above. Our case studies provide clearly defined, relevant scenarios, which can be used as starting points to study \gls{llm} behavior theoretically.

\subsection{Case Study: Non-Identifiability of rule extrapolation}
\label{sec:case-study-1}

Consider training an \gls{ar} language model $q$ to fit samples from the probabilistic context-free grammar (\gls{pcfg}) $p$ over sequences of the form $a^nb^n$ where $n$ is random. Such a distribution has limited support since there are ungrammatical sequences that occur with probability $0$. Moreover, there exist finite length prefixes $x_{1:l}$ which cannot be completed grammatically, whose marginal probability $p(x_{1:l})$ is thus $0$. We refer to such prefixes as \gls{ood} prompts. We say that $q$ generalizes perfectly (in the statistical sense) if $\kl{p(x_{1:k})}{q(x_{1:k})}=0,$ for some length $k$, \ie, it achieves maximal likelihood both in training and test. If $x_{1:l}$ has zero probability under $p$, the completion distribution $p(x_{l+1:k} \vert x_{1:l})$ is undefined. However, $q$ still defines a distribution over completions $q(x_{l+1:k} \vert x_{1:l})$. Since the \gls{kld} divergence is insensitive to the choice of $q(x_{l+1:k} \vert x_{1:l})$, the completion distribution is non-identifiable.

This means that any property of $q$ that depends only on completions of \gls{ood} prompts is non-identifiable. The $a^nb^n$ grammar can be described as the intersection of two rules:
 \begin{enumerate}[nolistsep,label=(R\arabic*),leftmargin=35pt]
    \item the number of $a$s and $b$s match; and \label{rule1}
    \item  $a$ never follows a $b$.\label{rule2}
\end{enumerate}

Unless a prompt can be completed consistently with both rules, the behaviour of $q$ is non-identifiable. It is meaningful to ask whether a trained model $q$ still respects rule R1 when completing \gls{ood} prompts that break rule R2, such as $abaa$. We call this \emph{rule extrapolation}, illustrated in Fig.\  \ref{fig:case-study-1}.

\begin{tcolorbox}[colback=gray!15!white,colframe=gray!75!black, title={ OOD rule extrapolation is non-identifiable, LLMs still extrapolate due to inductive biases}]
    Trained Transformers have better-than-chance ability to extrapolate rules on \gls{ood} prompts in a zero-shot manner (chance is zero, calculated at initialization, \cf \cref{fig:case-study-1,figure:rule_extrapolation,tab:extrapolation}). Since the \gls{kld} is insensitive to rule extrapolation, we attribute rule extrapolation to inductive biases.
\end{tcolorbox}

\textbf{Empirical demonstration.}
We train a decoder-only Transformer~\citep{vaswani_attention_2017_edit,radford_language_2018} on the  $a^nb^n$ \gls{pcfg} and evaluate zero-shot rule extrapolation (\cref{figure:rule_extrapolation}), measured as the proportion of times \gls{ood} prompts of length $8$ are completed consistently with rule R1.

As hypothesized, we find that the model develops an ability to extrapolate the rule in $43.7\%$ of the cases, although the loss function is agnostic to this. To illustrate this last point, we train two additional models in an adversarial, and an oracle setting, whereby an extra supervised loss term is added with either poisoned or helpful data. Consistently with our expectations, all three models reach the same, minimal test loss, but display widely varying rule extrapolation performance (\cref{figure:rule_extrapolation}; \cf \cref{sec:app_exp} for training details).\\
Our findings demonstrate that \gls{ar} probabilistic language models extrapolate \gls{ood} in meaningful ways. However, this behaviour is not merely a consequence of good generalisation, it arises as a result of additional inductive biases.

\begin{figure}[tb]
    \centering
    \includesvg[width=5.5cm,keepaspectratio]{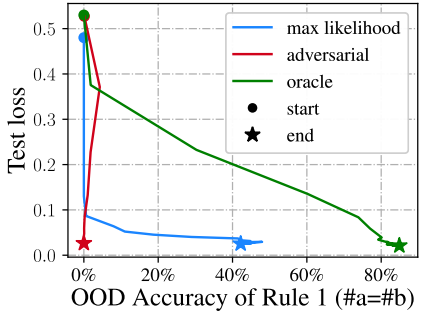}
    \caption{\textbf{\gls{ood} rule extrapolation in Transformers is  better than chance:}
     We trained a Transformer via {\color{figblue}{maximum likelihood}} on the $a^nb^n$ \gls{pcfg}. We evaluated the model on \gls{ood} prompts which are inconsistent with $a^nb^n$, and checked whether the completions obey rule \ref{rule1} ($x$ axis). Two other models, trained by an {\color{figred}{adversarial}} and an {\color{figgreen}{oracle}} process achieved the same test loss but displayed very different rule extrapolation accuracies. This demonstrates that test loss is insensitive to rule extrapolation behaviour and that the $43.7\%$  rule extrapolation accuracy (averaged over 20 seeds; details in \cref{sec:app_exp}) results from inductive biases.} 
    \label{figure:rule_extrapolation}
\end{figure}

\subsection{Case Study: \epsnonident and \acrfull{icl}}
\label{sec:case-study-2}

    In~\cref{sec:case-study-1}, we assumed that the pre-training and test distributions have limited support, \ie, there are \gls{ood} sequences with exactly $0$ probability under the pre-training distribution. A more realistic scenario is where some prompts have a non-zero but vanishingly small probability under the pre-training distribution. With full support, when non-zero probability is placed on all sequences, \gls{ar} probability distributions are identifiable. However, relaxing the strict definition of identifiability and considering models near-equivalent if their test performance is barely distinguishable with the \gls{kld}, we still find near-equivalent models that may behave radically differently on low-probability sequences, despite having access to infinite data. 
    We call this \epsnonident (for some small $\varepsilon>0$) and define it informally (\cf \cref{sec:app_ident}):
    \begin{definition}[\epsnonident of distributional properties (informal)]\label{definition:epsnonident_informal}
        A distributional property of $p$ is $\varepsilon-$non-identifiable if there is a distribution $q$ such that ${\kl{p}{q} \leq \varepsilon}$, but $q$ does not have the property of $p$.
    \end{definition}
    Contrary to traditional definitions, ours relaxes the distributional equivalence by admitting a non-zero \gls{kld}, and is formulated about having a property (\eg, \gls{icl}). This distinction might seem subtle, yet is important since in practice, the goal is to have a well-performing model; minimizing a loss can be insufficient~\citep{liu_same_2023,saunshi_understanding_2022,rusak_content_2022_edit,tay_scale_2022}. Indeed, this observation is widely accepted in evaluating generative models~\citep[20.4]{murphy_probabilistic_2022}, where the trade-off between the likelihood, representation, and sample quality is well known~\citep{inferenceMaximumLikelihood}.

    \begin{tcolorbox}[colback=gray!15!white,colframe=gray!75!black, title={\Acrfull{icl} can be non-identifiable, LLMs still can do ICL}]
        When the pre-training distribution is a mixture of \glspl{hmm} as in \citet{in_context_bayesian}, we demonstrate that \gls{icl} is $\varepsilon-$non-identifiable, even with infinite data, due to the insensitivity of the \gls{kld}. However, \glspl{llm} can be still in-context learners; we need to understand why.
    \end{tcolorbox}

\textbf{Example: \epsnonident of \gls{icl}.}
    \begin{figure}[t]
        \centering
        \tikzset{every picture/.style={line width=0.75pt}} %

\begin{tikzpicture}[x=0.75pt,y=0.75pt,yscale=-.8,xscale=.8]
\draw  [color={rgb, 255:red, 0; green, 0; blue, 0 }  ,draw opacity=1 ][fill={rgb, 255:red, 221; green, 218; blue, 218 }  ,fill opacity=0.59 ] (120.73,45.99) .. controls (140.72,27.74) and (334.55,22.62) .. (369.2,26.28) .. controls (403.86,29.94) and (339.27,124.31) .. (310.71,113.34) .. controls (282.15,102.36) and (138.01,130.16) .. (126.57,102.36) .. controls (115.13,74.56) and (100.73,64.24) .. (120.73,45.99) -- cycle ;
\draw  [fill={rgb, 255:red, 221; green, 218; blue, 218 }  ,fill opacity=0.59 ] (104.84,145.47) -- (151.68,145.47) -- (151.68,165.81) -- (104.84,165.81) -- cycle ;
\draw  [fill={rgb, 255:red, 255; green, 163; blue, 28 }  ,fill opacity=0.43 ] (330.84,145.93) -- (377.68,145.93) -- (377.68,166.27) -- (330.84,166.27) -- cycle ;
\draw  [fill={rgb, 255:red, 128; green, 128; blue, 128 }  ,fill opacity=0.59 ] (218.95,145.4) -- (265.8,145.4) -- (265.8,165.74) -- (218.95,165.74) -- cycle ;
\draw  [fill={rgb, 255:red, 0; green, 0; blue, 0 }  ,fill opacity=1 ] (131.57,87.25) .. controls (131.57,85.56) and (132.81,84.19) .. (134.35,84.19) .. controls (135.88,84.19) and (137.13,85.56) .. (137.13,87.25) .. controls (137.13,88.93) and (135.88,90.3) .. (134.35,90.3) .. controls (132.81,90.3) and (131.57,88.93) .. (131.57,87.25) -- cycle ;
\draw  [color={rgb, 255:red, 0; green, 0; blue, 0 }  ,draw opacity=1 ][fill={rgb, 255:red, 128; green, 128; blue, 128 }  ,fill opacity=0.59 ] (178.5,62.45) .. controls (193,46.73) and (176.1,47.15) .. (253,42.17) .. controls (329.9,37.19) and (357.62,95.2) .. (289,90.17) .. controls (220.38,85.15) and (251,108.17) .. (228,102.17) .. controls (205,96.17) and (164,78.17) .. (178.5,62.45) -- cycle ;
\draw  [color={rgb, 255:red, 0; green, 0; blue, 0 }  ,draw opacity=1 ][fill={rgb, 255:red, 255; green, 163; blue, 28 }  ,fill opacity=0.43 ] (218,115) .. controls (241,104.04) and (218,56.33) .. (237,51.33) .. controls (252.61,47.22) and (259.45,54.59) .. (255.84,76.77) .. controls (255.06,81.58) and (253.78,87.09) .. (252,93.33) .. controls (242,128.33) and (184,155) .. (170,133) .. controls (156,111) and (195,125.96) .. (218,115) -- cycle ;
\draw  [fill={rgb, 255:red, 0; green, 0; blue, 0 }  ,fill opacity=1 ] (238.57,79.25) .. controls (238.57,77.56) and (239.81,76.19) .. (241.35,76.19) .. controls (242.88,76.19) and (244.13,77.56) .. (244.13,79.25) .. controls (244.13,80.93) and (242.88,82.3) .. (241.35,82.3) .. controls (239.81,82.3) and (238.57,80.93) .. (238.57,79.25) -- cycle ;

\draw (267.8,148.8) node [anchor=north west][inner sep=0.75pt]    {$KL=0$};
\draw (153.68,148.87) node [anchor=north west][inner sep=0.75pt]    {$KL\leq \varepsilon $};
\draw (379.68,148.93) node [anchor=north west][inner sep=0.75pt]   [align=left] {$\displaystyle ICL$};
\draw (131.8,63.8) node [anchor=north west][inner sep=0.75pt]    {$q$};
\draw (239,54.73) node [anchor=north west][inner sep=0.75pt]    {$p$};

\end{tikzpicture}
        \caption{\textbf{Vanishingly small \gls{kld} cannot capture \acrfull{icl}:}
        illustration of \cref{our_theorem}, showing that when $p$ displays \gls{icl} property, there exists a distribution $q$ that is $\varepsilon-$close in KL divergence, which has no \gls{icl} ability.}
        \label{fig:in_context}
    \end{figure}
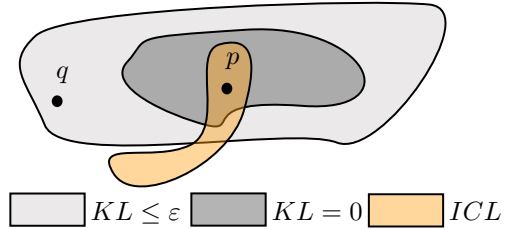
    \gls{icl} refers to the ability of a model to learn a downstream task based on a prompt consisting of input-output examples. That is, if the prompt includes sufficient input-output pairs $( x_i, y_i )$ and a test input $x_{\text{test}}$ then the model can predict the correct output $y_{\text{test}}$ without fine-tuning, only relying on the pairs $( x_i, y_i )$ in inferring the task. 

    Numerous theories exist in the literature, which prove the emergence of \gls{icl} under simplified settings. For a Transformer consisting of linear self-attention layers trained on regression data, \gls{icl} is shown to emerge as implicit gradient descent \citep{tengyumaiclmeta, von_oswald_transformers_2023, tengyumaICLmetaloss}. This view interprets Transformers as meta-models that ``learn to learn'' in-context tasks within their forward pass. The framework requires specific model parameters. Other approaches prove \gls{icl} using assumptions on the structure of the training data ~\citep{in_context_bayesian, wang2023large}.
   We highlight an unmentioned type of fragility of these theories for real-world \glspl{llm}, which occurs in practice as soon as we slightly deviate from these assumptions. 
    In the setting of \citep{in_context_bayesian}, we prove that \gls{icl} can be $\varepsilon-$non-identifiable, i.e., for a distribution $p$, with \gls{icl} ability by construction, we construct another distribution $q$ within $\varepsilon$ \gls{kld} divergence that provably does not exhibit \gls{icl}. 

    Our result demonstrates that the emergence of \gls{icl} is not a direct consequence of minimizing the negative log-likelihood. Indeed, an \gls{llm} that fits the pre-training distribution within $\varepsilon$ \gls{kld} divergence,
    can easily not exhibit \gls{icl}. We detail the implications for the saturation regime in \cref{appx:bicikli}.

 \citet{in_context_bayesian} demonstrate that for a mixture of \glspl{hmm} pre-training distribution $p$, the \gls{llm} is an in-context learner in the limit of infinite examples in the prompt. That is, it produces completions aligning with the predictions of the prompt distribution. In accordance with their notation, let $p_0(\theta)$ be a prior distribution on the latent concepts $\theta \in \Theta$, and for each $\theta$ let the distribution $p(o_1, \dots, o_T | \theta)$ be a \gls{hmm} for a token sequence $o_1, \dots, o_T
 $ of a fixed pre-training document length $T$. First, a concept $\theta$ is sampled from $p_0(\theta)$, then a document is generated according to $p(o_1, \dots, o_T | \theta)$.  
 Therefore, the pre-training distribution (on token sequences) is a mixture of \glspl{hmm} induced by these conditionals.
    \begin{equation}
        p (o_1, \dots, o_T) = \int_{\theta \in \Theta} p(o_1, \dots, o_T | \theta) p_0(\theta) \, d \theta.
    \end{equation}
The prompt for in-context learning is a concatenation of $n$ independent input-output pairs $(x_i, y_i)$ and a test example $x_{\text{test}}$ each separated by a special delimiter token $o^{\text{delim}}$. Each example sampled from $p(\cdot|\theta^*)$, where $\theta^*$ is the ground-truth concept. The distribution of the concatenated examples are called the prompt distribution $p_{\text{prompt}}$ 
\begin{flalign}
     ( S_n, x_{\text{test}}) &= (x_1, y_1, o^{\text{delim}}, x_2, y_2, o^{\text{delim}},\nonumber\\
     &\dots, x_n, y_n, o^{\text{delim}}, x_{\text{test}}) \sim p_{\text{prompt}}.
\vspace{-0.5em}
\end{flalign}

    In \gls{icl}, the goal is to predict the test output $y_{\text{test}}$ by predicting the next token. $y_{\text{test}}$ is defined by $\argmax_y p_{\text{prompt}} (y|x_{\text{test}})$. Details on the notation and the model are in \cref{appx:framework}.\\
    We say that the model is an in-context learner
    if, as $n\! \rightarrow\! \infty,$ 
    \begin{equation}
    \label{incontext}
        \argmax_y p(y|S_n, x_{\text{test}}) \rightarrow \argmax_y p_{\text{prompt}} (y|x_{\text{test}}).
    \end{equation}
    
\vspace{-0.5em}    \citet{in_context_bayesian} proved that if certain assumptions hold (\cref{appx:framework}) and the pre-training distribution is a mixture of \glspl{hmm}, then \gls{icl} occurs when the number of examples $n$ is large enough. Thus,
    the two sides of \eqref{incontext} become equal.
    Let $n_0$ denote the threshold example sequence length, after which $p$ satisfies in-context learning. 
    However, as we show, matching the pre-training distribution up to $\varepsilon>0$ \gls{kld} cannot guarantee \gls{icl}, even with increasing prompt size.
    \begin{restatable}{proposition}{ourtheorem}[\epsnonident of \gls{icl}]\label{our_theorem}
        Let $N$ be the length of a prompt $(S_n, x_{\text{test}})$. For all $\varepsilon >0$, there exists $n_1 \geq n_0$, such that for all $n \geq n_1$, there exists a distribution $q_n$ close to a mixture of \glspl{hmm} $p$ in \gls{kld} divergence 
        \begin{align*}
            \kl{ p(o_1, \dots, o_N)}{ q_n(o_1, \dots, o_N)}\leq \varepsilon,\ \text{s.t.}\\
            \argmax_y q_n(y|S_n, x_{\text{test}}) \neq \argmax_y p_{\text{prompt}} (y|x_{\text{test}}).
        \end{align*}
    \end{restatable}

    \begin{proof}[Proof (Sketch)]
        We construct a distribution $q_n$ that matches $p$ everywhere, except for a few sequences ending with a prompt structure. 
        Changing the highest and second-highest probabilities changes the output of the $\argmax$.
        Then we bound \kl{p}{q_n} and exploit that almost all conditionals are the same (except those we changed).
        Since the probability of prompts goes to zero as their length $n\to \infty$, we conclude that the \gls{kld} converges to zero. The proof is in \cref{appx:proof}.
    \end{proof}

\subsection{Case Study: Parameter Non-identifiability and Fine-tuning}
\label{sec:case-study-3}
\begin{figure}[t]
    \centering
    \tikzset{every picture/.style={line width=0.75pt}} %

\begin{tikzpicture}[x=0.75pt,y=0.75pt,yscale=-.65,xscale=.9]
\draw  [fill={rgb, 255:red, 221; green, 218; blue, 218 }  ,fill opacity=0.3 ] (292.81,25) -- (440.5,25) -- (440.5,202) -- (292.81,202) -- cycle ;
\draw  [fill={rgb, 255:red, 221; green, 218; blue, 218 }  ,fill opacity=0.3 ] (118.5,25) -- (255.26,25) -- (255.26,202) -- (118.5,202) -- cycle ;
\draw  [fill={rgb, 255:red, 221; green, 218; blue, 218 }  ,fill opacity=1 ] (320.77,52.57) .. controls (324,43.83) and (342.64,44.62) .. (362.41,54.35) .. controls (382.17,64.08) and (395.58,79.06) .. (392.34,87.81) .. controls (389.11,96.56) and (370.47,95.76) .. (350.7,86.03) .. controls (330.94,76.3) and (317.53,61.32) .. (320.77,52.57) -- cycle ;
\draw [color={rgb, 255:red, 74; green, 97; blue, 226 }  ,draw opacity=1 ][line width=1.5]    (187.37,72.1) .. controls (183.53,69.91) and (177.17,56.61) .. (172.54,68.84) .. controls (166.98,80.04) and (165.26,52.27) .. (161.41,61.7) .. controls (160.37,64.75) and (164.96,94.19) .. (158.5,88) .. controls (150.71,80.47) and (146.5,60) .. (138.5,72) .. controls (134.5,82) and (135.01,85.49) .. (130.29,80.52) .. controls (121.46,72.07) and (128.87,93.43) .. (132.57,103.68) ;
\draw [shift={(133.53,106.36)}, rotate = 250.19] [color={rgb, 255:red, 74; green, 97; blue, 226 }  ,draw opacity=1 ][line width=1.5]    (9.95,-2.99) .. controls (6.32,-1.27) and (3.01,-0.27) .. (0,0) .. controls (3.01,0.27) and (6.32,1.27) .. (9.95,2.99)   ;
\draw  [fill={rgb, 255:red, 74; green, 97; blue, 226 }  ,fill opacity=1 ] (130.75,109.41) .. controls (130.75,107.73) and (132,106.36) .. (133.53,106.36) .. controls (135.07,106.36) and (136.31,107.73) .. (136.31,109.41) .. controls (136.31,111.1) and (135.07,112.47) .. (133.53,112.47) .. controls (132,112.47) and (130.75,111.1) .. (130.75,109.41) -- cycle ;
\draw [color={rgb, 255:red, 245; green, 163; blue, 28 }  ,draw opacity=1 ][line width=1.5]    (201.49,146.96) .. controls (197.65,144.77) and (192.22,133.72) .. (187.59,145.94) .. controls (182.02,157.15) and (184.35,150.57) .. (180.5,160) .. controls (175.86,172.23) and (171.83,150.02) .. (162.56,147.98) .. controls (155.5,156) and (159.43,154.83) .. (158.5,164) .. controls (158.5,175.21) and (150.84,167.96) .. (142.5,170) .. controls (134.14,173.34) and (136.13,163.81) .. (137.32,158.25) ;
\draw [shift={(137.87,155.35)}, rotate = 95.76] [color={rgb, 255:red, 245; green, 163; blue, 28 }  ,draw opacity=1 ][line width=1.5]    (9.95,-2.99) .. controls (6.32,-1.27) and (3.01,-0.27) .. (0,0) .. controls (3.01,0.27) and (6.32,1.27) .. (9.95,2.99)   ;
\draw  [fill={rgb, 255:red, 245; green, 163; blue, 28 }  ,fill opacity=1 ] (135.15,152.48) .. controls (135.15,150.9) and (136.37,149.62) .. (137.87,149.62) .. controls (139.37,149.62) and (140.59,150.9) .. (140.59,152.48) .. controls (140.59,154.06) and (139.37,155.35) .. (137.87,155.35) .. controls (136.37,155.35) and (135.15,154.06) .. (135.15,152.48) -- cycle ;
\draw [color={rgb, 255:red, 214; green, 16; blue, 16 }  ,draw opacity=1 ][line width=1.5]    (187.37,72.1) .. controls (209.49,44.87) and (300.67,60.58) .. (325.02,138.3) ;
\draw [shift={(325.73,140.67)}, rotate = 253.77] [color={rgb, 255:red, 214; green, 16; blue, 16 }  ,draw opacity=1 ][line width=1.5]    (14.21,-4.28) .. controls (9.04,-1.82) and (4.3,-0.39) .. (0,0) .. controls (4.3,0.39) and (9.04,1.82) .. (14.21,4.28)   ;
\draw [color={rgb, 255:red, 214; green, 16; blue, 16 }  ,draw opacity=1 ][line width=1.5]    (200.35,145.25) .. controls (221.63,107.11) and (277.67,107.71) .. (323.64,141.21) ;
\draw [shift={(325.73,142.77)}, rotate = 217.08] [color={rgb, 255:red, 214; green, 16; blue, 16 }  ,draw opacity=1 ][line width=1.5]    (14.21,-4.28) .. controls (9.04,-1.82) and (4.3,-0.39) .. (0,0) .. controls (4.3,0.39) and (9.04,1.82) .. (14.21,4.28)   ;
\draw  [fill={rgb, 255:red, 0; green, 0; blue, 0 }  ,fill opacity=1 ] (197.57,145.25) .. controls (197.57,143.56) and (198.81,142.19) .. (200.35,142.19) .. controls (201.88,142.19) and (203.13,143.56) .. (203.13,145.25) .. controls (203.13,146.93) and (201.88,148.3) .. (200.35,148.3) .. controls (198.81,148.3) and (197.57,146.93) .. (197.57,145.25) -- cycle ;
\draw  [fill={rgb, 255:red, 0; green, 0; blue, 0 }  ,fill opacity=1 ] (184.59,72.1) .. controls (184.59,70.41) and (185.83,69.04) .. (187.37,69.04) .. controls (188.91,69.04) and (190.15,70.41) .. (190.15,72.1) .. controls (190.15,73.79) and (188.91,75.16) .. (187.37,75.16) .. controls (185.83,75.16) and (184.59,73.79) .. (184.59,72.1) -- cycle ;
\draw [color={rgb, 255:red, 245; green, 163; blue, 28 }  ,draw opacity=1 ][line width=1.5]    (325.73,146.4) .. controls (327.69,149.74) and (326.1,163.07) .. (333.36,156.18) .. controls (340.92,150.46) and (339.34,166.2) .. (344.58,161.19) .. controls (349.82,157.01) and (349.62,147.25) .. (352.1,154.72) .. controls (355.06,163.77) and (352.11,163.54) .. (359.54,157.85) .. controls (365.92,152.04) and (358.25,163.31) .. (365.71,160.24) .. controls (371.94,157.03) and (377.33,166.15) .. (380.4,171.36) ;
\draw [shift={(381.93,173.85)}, rotate = 235.29] [color={rgb, 255:red, 245; green, 163; blue, 28 }  ,draw opacity=1 ][line width=1.5]    (9.95,-2.99) .. controls (6.32,-1.27) and (3.01,-0.27) .. (0,0) .. controls (3.01,0.27) and (6.32,1.27) .. (9.95,2.99)   ;
\draw  [fill={rgb, 255:red, 0; green, 0; blue, 0 }  ,fill opacity=1 ] (323.3,143.53) .. controls (323.3,141.95) and (324.39,140.67) .. (325.73,140.67) .. controls (327.08,140.67) and (328.17,141.95) .. (328.17,143.53) .. controls (328.17,145.12) and (327.08,146.4) .. (325.73,146.4) .. controls (324.39,146.4) and (323.3,145.12) .. (323.3,143.53) -- cycle ;
\draw  [fill={rgb, 255:red, 245; green, 163; blue, 28 }  ,fill opacity=1 ] (379.84,177.04) .. controls (379.66,175.47) and (380.6,174.05) .. (381.93,173.85) .. controls (383.27,173.66) and (384.49,174.78) .. (384.66,176.35) .. controls (384.84,177.92) and (383.9,179.34) .. (382.57,179.53) .. controls (381.24,179.72) and (380.01,178.61) .. (379.84,177.04) -- cycle ;
\draw [color={rgb, 255:red, 74; green, 97; blue, 226 }  ,draw opacity=1 ][line width=1.5]    (327.58,141.85) .. controls (331.43,140.09) and (338.01,143.33) .. (332.98,131.54) .. controls (329.2,119.65) and (345.14,131.02) .. (341.44,121.7) .. controls (337.32,112.08) and (334.54,109.46) .. (342.64,108.12) .. controls (350.73,111.82) and (359.18,110.08) .. (345.5,98) .. controls (341.33,87.71) and (360.5,111) .. (359.5,97) .. controls (347.88,96.61) and (353.46,83.44) .. (357.57,75.66) ;
\draw [shift={(358.99,73.06)}, rotate = 119.26] [color={rgb, 255:red, 74; green, 97; blue, 226 }  ,draw opacity=1 ][line width=1.5]    (9.95,-2.99) .. controls (6.32,-1.27) and (3.01,-0.27) .. (0,0) .. controls (3.01,0.27) and (6.32,1.27) .. (9.95,2.99)   ;
\draw    (394.98,109.29) -- (387.27,94.65) -- (378.28,77.58) ;
\draw [shift={(377.34,75.81)}, rotate = 62.22] [color={rgb, 255:red, 0; green, 0; blue, 0 }  ][line width=0.75]    (10.93,-3.29) .. controls (6.95,-1.4) and (3.31,-0.3) .. (0,0) .. controls (3.31,0.3) and (6.95,1.4) .. (10.93,3.29)   ;
\draw  [fill={rgb, 255:red, 74; green, 97; blue, 226 }  ,fill opacity=1 ] (356.55,70.19) .. controls (356.55,68.61) and (357.64,67.33) .. (358.99,67.33) .. controls (360.33,67.33) and (361.42,68.61) .. (361.42,70.19) .. controls (361.42,71.77) and (360.33,73.06) .. (358.99,73.06) .. controls (357.64,73.06) and (356.55,71.77) .. (356.55,70.19) -- cycle ;

\draw (188.44,72.67) node [anchor=north west][inner sep=0.75pt]  [xscale=0.8,yscale=0.8]  {$\theta _{1}$};
\draw (139.6,97.8) node [anchor=north west][inner sep=0.75pt]  [color={rgb, 255:red, 76; green, 97; blue, 226 }  ,opacity=1 ,xscale=0.8,yscale=0.8]  {$\theta '_{1}$};
\draw (201.42,145.82) node [anchor=north west][inner sep=0.75pt]  [xscale=0.8,yscale=0.8]  {$\theta _{2}$};
\draw (138.54,132.65) node [anchor=north west][inner sep=0.75pt]  [xscale=0.8,yscale=0.8]  {$\textcolor[rgb]{0.96,0.64,0.11}{\theta '}\textcolor[rgb]{0.96,0.64,0.11}{_{2}}$};
\draw (310.41,147.45) node [anchor=north west][inner sep=0.75pt]  [xscale=0.8,yscale=0.8]  {$P$};
\draw (337.5,55.28) node [anchor=north west][inner sep=0.75pt]  [xscale=0.8,yscale=0.8]  {$\textcolor[rgb]{0.3,0.38,0.89}{P'}\textcolor[rgb]{0.3,0.38,0.89}{_{1}}$};
\draw (382.85,172.99) node [anchor=north west][inner sep=0.75pt]  [xscale=0.8,yscale=0.8]  {$\textcolor[rgb]{0.96,0.64,0.11}{P'}\textcolor[rgb]{0.96,0.64,0.11}{_{2}}$};
\draw (120.86,29.6) node [anchor=north west][inner sep=0.75pt]  [xscale=0.8,yscale=0.8] [align=left] {{\footnotesize parameter space}};
\draw (301.92,29.06) node [anchor=north west][inner sep=0.75pt]  [font=\footnotesize,xscale=0.8,yscale=0.8] [align=left] {distribution space};
\draw (352.55,111.32) node [anchor=north west][inner sep=0.75pt]  [font=\footnotesize,xscale=0.8,yscale=0.8] [align=left] {good downstream\\performance};

\end{tikzpicture}
    \caption{\textbf{Illustration of parameter non-identifiability:} Two sets of parameters ($\theta_1, \theta_2$) may describe the same AR \gls{llm} and thus achieve the same test loss and perform identically in benchmarks. When fine-tuned on the same data, parameter-dependent inductive biases may push the two models apart, and it is possible that, say, $\theta_1$ enables significantly more data-efficient fine-tuning than $\theta_2$.}
    \label{fig:case_study_3}
\end{figure}
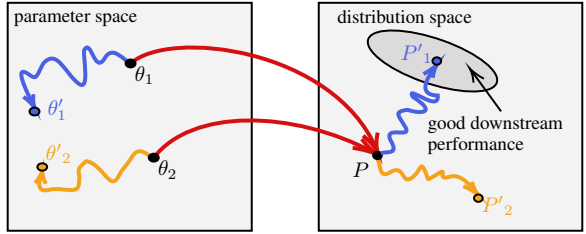

    A neural network's parametrization affects its learning dynamics~\citep{saxe_exact_2014,jacot_neural_2020,dinh_sharp_2017}, which implies \textit{parameter non-identifiability} (\cref{fig:case_study_3}), \ie, functionally equivalent models can behave differently during fine-tuning and transfer\footnote{\textit{Fine-tuning} is training a pre-trained model with a new loss, or with 
    a new data set; \textit{transfer} is when the model can also change after pre-training, \eg, by applying a readout layer.}.
    Parameter non-identifiability is relevant in \glspl{llm}, since large pre-trained models are often fine-tuned on new data sets to solve specialized tasks. For example, given the same pre-training loss, a narrow and wide Transformer can perform better downstream than a wide and shallow one~\cite{tay_scale_2022}.
    \citet{liu_same_2023} demonstrates parameter non-identifiability in a clever experiment: they ``embed" a small Transformer into a larger one by maintaining functional equivalence and demonstrate that the different architectural constraints of the larger model interact differently with the optimization method: despite having the same pre-training loss, the optimization will not prefer the embedded Transformer, but one with flatter minima, yielding a $10\%$ difference in downstream accuracy after fine-tuning. 
    This example highlights the need to understand what parametrizations are useful for improving fine-tuning and transfer in \glspl{llm}.
    A prominent method for \gls{llm} fine-tuning is \gls{rlhf}; recent results suggest that its success on \gls{ood} data might be due to an inductive bias. The concurrent study by \citet{kirk_understanding_2024} empirically demonstrated better \gls{ood} performance than supervised fine-tuning---at the cost of reduced diveristy of the generated completions. The authors reason, based on arguments from~\citep{goldberg_reinforcement_2023,xu_generalization_2022} that the \gls{rlhf} objective has an inductive bias: namely, it forces the \gls{llm} to optimize its policy for \gls{ood} data.
   
\begin{tcolorbox}[colback=gray!15!white,colframe=gray!75!black, title={We need to understand how inductive biases in LLMs affect fine-tuning performance}]
    Functionally equivalent models can have different parametrizations. These are indistinguishable by the test loss but may differ after fine-tuning in downstream performance, as inductive biases in different parametrizations affect gradient dynamics (\cref{fig:case_study_3}).
\end{tcolorbox}

\section{The saturation regime}
\label{sec:position}

In \cref{sec:implications}, we presented case studies to support our position that statistical generalization cannot explain the most interesting phenomena in \glspl{llm}. We believe that we need a paradigm shift akin to that brought about by the discovery of the interpolation regime. Our attention should be directed towards the regime where we can assume our model achieves both near-zero train and test loss. We should study inductive biases that give rise to beneficial qualities, not implied by statistical generalization.\\
The term saturation regime was coined by \citet{liu_same_2023} in the context of transfer learning to argue that a minimal test pre-training loss does not guarantee the transfer performance observed in practice (\cref{sec:case-study-3}). We adopt this name to study a broader set of phenomena, \eg zero-shot \gls{ood} extrapolation (\cref{sec:case-study-1}) and \gls{icl} (\cref{sec:case-study-2}).\\
We focus on \gls{ar} \glspl{llm}, which are in the intersection of probabilistic generative models and \gls{ssl}/transfer learning methods. As our case studies show, all the questions raised by non-identifiability, are relevant and interesting for \gls{ar} \glspl{llm}---\eg, \gls{icl} or zero-shot prompting cannot be interpreted for non-\gls{ar} \glspl{llm} such as BERT~\citep{devlin2019bert}, though those can also operate in the saturation regime. Some of our case studies, e.g. transfer (\cref{sec:case-study-3}) remain relevant in a much wider context, including vision models in \gls{ssl}, where theoretical studies of relevant inductive biases already exist~\citep{haochen_theoretical_2022}.

\begin{tcolorbox}[colback=figblue!15!white,colframe=figblue!75!black, title={\textbf{Position:} we need to study LLMs in the saturation regime}]
In this regime, the relevant questions are 
    \begin{enumerate}[nolistsep,leftmargin=*,label=(\roman*)]
        \item assessing \glspl{llm} with better generalization measures;
        \item understanding transferability; and
        \item studying the inductive biases enabling them. 
    \end{enumerate} 
\end{tcolorbox}
We highlight our arguments for these research directions, then review relevant research in \cref{sec:discussion}. We also suggest extensions for identifiability to capture the desired properties (generalization, transfer) and inductive biases.

    \textbf{(i) Generalization measures.} In the saturation regime, the loss cannot capture the structures and properties of interest (such as the structure of natural languages); thus, we advocate for studying other (\gls{ood}) generalization measures: compositional, systematic, and symbolic generalization (\cref{better-gen-measures}). We argue that due to their well-defined structure, formal languages such as \glspl{pcfg} are ideal for studying these generalization measures in \glspl{llm}.

    \textbf{(ii) Transferability.} The size of state-of-the-art \glspl{llm} prohibits training them from scratch, making transferability crucial. Understanding and controlling transferability in \glspl{llm} is key to building strong general-purpose models and preventing them from being fine-tuned for harmful tasks \citep{qi2023finetuning_harmful}. Since current metrics cannot reliably capture transfer performance (\cf \cref{sec:case-study-3}), we suggest developing suitable metrics for fine-tuning abilities in \glspl{llm}.

   \textbf{(iii) Inductive biases} are widely researched in machine learning, \eg, flatness \citep{sam}, gradient noise \citep{jiang2019fantastic_beasts}, information geometry \citep{Hessian, gFg} and simplicity bias \cite{simplicitybiasSGD}. As opposed to computer vision~\citep{klindt_towards_2021,geirhos_generalisation_2018,geirhos_shortcut_2020,geirhos_imagenet-trained_2022,torok_tracking_2022,offert_understanding_2021,goyal_inductive_2022,papa_inductive_2022}, it is unclear what kind of inductive biases are useful for natural languages, especially for more complex tasks such as reasoning.
   Since the training and test losses do not indicate the desired properties, inductive biases offer alternative means to ensure good \gls{ood} performance (\cref{sec:case-study-1}) and transfer (\cref{sec:case-study-3}). 
   For \glspl{llm}, we advocate focusing on inductive biases inducing specific model qualities either in weight or function space (\cref{subsec:disc_indctive_biases}), as these enable us to infer \gls{ood} and transfer performance, irrespective of the loss value.

    \textbf{Lessons for identifiability research.}
    Identifiability theory mostly focuses on how functional, 
    ~\citep{shimizu_linear_2006,hoyer_nonlinear_2008,lachapelle_gradient-based_2020,gresele_independent_2021} and distributional
    ~\citep{hyvarinen_nonlinear_1999,hyvarinen_unsupervised_2016,guo_causal_2022,ke_learning_2020} assumptions can guarantee identifiability, almost exclusively assuming \acrshort{iid} data, with a few pioneering works entering the non-\acrshort{iid} realm by studying compositionality~\citep{brady_provably_2023,wiedemer_compositional_2023,wiedemer_provable_2023,lachapelle_additive_2023} or exchangeability~\citep{guo_causal_2022}. We must go beyond the \acrshort{iid} assumption to understand \glspl{llm} and, we argue, there are useful lessons for identifiability theory to be drawn from studying the saturation regime, \eg, by considering that:
    \begin{enumerate}[nolistsep,label=(\roman*)]
        \item studying the identifiability of task-specific properties, \eg, \gls{icl} (\cf \cref{sec:case-study-2}), as opposed to focusing on identifying all aspects of an underlying distribution;
        \item inductive biases play a role in training non-identifiable models, \ie, some non-identifiable properties might become identifiable once we consider which solutions are reachable by the optimization algorithm;
        \item models in the equivalence class of global minima are reached with different probabilities by training; inductive biases are akin to a prior distribution over models;
        \item non-identifiability may be a feature, not a bug when the \acrshort{mle} objective is misspecified and does not fully describe how the models are later used, but useful inductive biases are present that nudge the model towards a useful solution (\cref{fig:prior}).
    \end{enumerate}

\section{Discussion: where next?}
\label{sec:discussion}

Our main message is highlighting the need to study \glspl{llm} in the saturation regime, where they achieve perfect (statistical) generalization, but this does not guarantee the presence of practical model properties we seek.
These properties are OOD generalization and transferability, and we attribute them to qualitative, non problem-specific inductive biases (\cref{sec:position}).
Here we detail promising research directions and formulate concrete research questions for studying \glspl{llm} in the saturation regime, focusing on generalization measures (\cref{better-gen-measures}), computational language modeling (\cref{subsec:disc_comp_lang_models}), and inductive biases (\cref{subsec:disc_indctive_biases}).

    \subsection{Better generalization measures}
    \label{better-gen-measures}

While statistical generalization offers a robust framework for generalization theory, bounding the test loss only does not explain important properties of \glspl{llm}, as illustrated in \cref{sec:implications}. 
In this section, we give examples of alternative generalization metrics that better align with the properties of natural languages. We advocate for adapting the tools of statistical generalization to develop formalizations for these metrics in the saturation regime. 
Some works have already started extending the classical PAC-Bayes framework to generative models~\citep{PAC_VAE, PAC_GAN}, meta-learning~\citep{PAC-meta, PAC-meta-Vincent}, and non-\acrshort{iid} data~\citep{PAC-time-series, PAC-domain-adapt, PAC-hostile}; however, further efforts are required. Results may also be established for quantities other than generalization error, perhaps bounding \gls{llm}-relevant properties, \eg, \gls{icl}.

    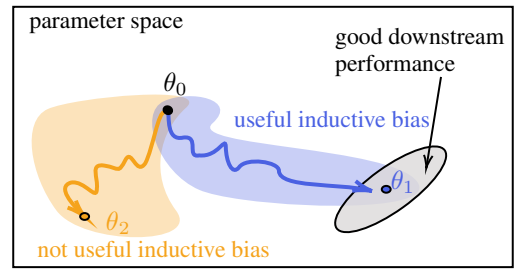
\begin{figure}[t]
        \centering
        \tikzset{every picture/.style={line width=0.75pt}} %

\begin{tikzpicture}[x=0.75pt,y=0.75pt,yscale=-.6,xscale=.67]
\draw  [draw opacity=0][fill={rgb, 255:red, 245; green, 163; blue, 28 }  ,fill opacity=0.3 ] (47.5,100.67) .. controls (67.5,90.67) and (171.5,79.67) .. (151.5,99.67) .. controls (131.5,119.67) and (124,156) .. (144,186) .. controls (164,216) and (74,216) .. (54,186) .. controls (34,156) and (27.5,110.67) .. (47.5,100.67) -- cycle ;
\draw  [draw opacity=0][fill={rgb, 255:red, 74; green, 97; blue, 226 }  ,fill opacity=0.3 ] (135,94) .. controls (155,84) and (191.5,92.67) .. (171.5,112.67) .. controls (151.5,132.67) and (305.5,131.67) .. (325.5,161.67) .. controls (341.29,185.36) and (249.24,185.36) .. (194.54,176.92) .. controls (179.95,174.68) and (168.02,171.83) .. (161.5,168.67) .. controls (130.5,153.67) and (115,104) .. (135,94) -- cycle ;
\draw  [fill={rgb, 255:red, 221; green, 218; blue, 218 }  ,fill opacity=0.8 ] (264.89,200.92) .. controls (259.76,192.56) and (273.82,171.14) .. (296.31,153.06) .. controls (318.8,134.99) and (341.2,127.11) .. (346.33,135.47) .. controls (351.47,143.82) and (337.4,165.25) .. (314.92,183.33) .. controls (292.43,201.4) and (270.03,209.28) .. (264.89,200.92) -- cycle ;
\draw [color={rgb, 255:red, 74; green, 97; blue, 226 }  ,draw opacity=1 ][line width=1.5]    (137.87,101.65) .. controls (139.17,106.69) and (136.43,128.01) .. (149.5,124.67) .. controls (162.5,122.86) and (151.21,141.99) .. (161.5,139.67) .. controls (175.92,141.41) and (181.24,122.6) .. (181.5,132.67) .. controls (181.75,144.86) and (186.25,156.98) .. (191.5,149.67) .. controls (196.89,141.06) and (221.65,145.03) .. (221.5,152.67) .. controls (221.87,167.18) and (265.77,163.05) .. (288.74,168.89) ;
\draw [shift={(291.5,169.67)}, rotate = 197.65] [color={rgb, 255:red, 74; green, 97; blue, 226 }  ,draw opacity=1 ][line width=1.5]    (14.21,-4.28) .. controls (9.04,-1.82) and (4.3,-0.39) .. (0,0) .. controls (4.3,0.39) and (9.04,1.82) .. (14.21,4.28)   ;
\draw [color={rgb, 255:red, 245; green, 163; blue, 28 }  ,draw opacity=1 ][line width=1.5]    (135.51,103.58) .. controls (130.78,101.39) and (126.63,134) .. (120.93,146.23) .. controls (114.09,157.43) and (110.84,133.74) .. (106.11,143.17) .. controls (100.4,155.4) and (101.54,150.3) .. (90.14,148.27) .. controls (78.74,146.23) and (85.58,158.45) .. (84.44,167.62) .. controls (84.44,178.83) and (75.32,160.49) .. (65.05,162.53) .. controls (53.76,166.2) and (66.42,181.42) .. (73.23,188.88) ;
\draw [shift={(75.24,191.06)}, rotate = 226.99] [color={rgb, 255:red, 245; green, 163; blue, 28 }  ,draw opacity=1 ][line width=1.5]    (14.21,-4.28) .. controls (9.04,-1.82) and (4.3,-0.39) .. (0,0) .. controls (4.3,0.39) and (9.04,1.82) .. (14.21,4.28)   ;
\draw  [fill={rgb, 255:red, 245; green, 163; blue, 28 }  ,fill opacity=1 ] (71.89,191.06) .. controls (71.89,189.47) and (73.39,188.19) .. (75.24,188.19) .. controls (77.08,188.19) and (78.58,189.47) .. (78.58,191.06) .. controls (78.58,192.64) and (77.08,193.92) .. (75.24,193.92) .. controls (73.39,193.92) and (71.89,192.64) .. (71.89,191.06) -- cycle ;
\draw  [fill={rgb, 255:red, 0; green, 0; blue, 0 }  ,fill opacity=1 ] (140.04,104.29) .. controls (138.74,105.36) and (136.71,105.05) .. (135.51,103.58) .. controls (134.31,102.12) and (134.4,100.07) .. (135.7,99) .. controls (137.01,97.93) and (139.04,98.25) .. (140.24,99.71) .. controls (141.43,101.17) and (141.35,103.22) .. (140.04,104.29) -- cycle ;
\draw    (343.5,73.67) -- (332.68,148.54) ;
\draw [shift={(332.39,150.52)}, rotate = 278.23] [color={rgb, 255:red, 0; green, 0; blue, 0 }  ][line width=0.75]    (10.93,-3.29) .. controls (6.95,-1.4) and (3.31,-0.3) .. (0,0) .. controls (3.31,0.3) and (6.95,1.4) .. (10.93,3.29)   ;
\draw   (21.5,14.67) -- (397.5,14.67) -- (397.5,233.67) -- (21.5,233.67) -- cycle ;
\draw  [fill={rgb, 255:red, 74; green, 97; blue, 226 }  ,fill opacity=1 ] (298.93,168.19) .. controls (298.93,166.61) and (300.42,165.33) .. (302.27,165.33) .. controls (304.12,165.33) and (305.61,166.61) .. (305.61,168.19) .. controls (305.61,169.77) and (304.12,171.06) .. (302.27,171.06) .. controls (300.42,171.06) and (298.93,169.77) .. (298.93,168.19) -- cycle ;

\draw (132.33,67.7) node [anchor=north west][inner sep=0.75pt]    {$\theta _{0}$};
\draw (303.38,149.79) node [anchor=north west][inner sep=0.75pt]  [color={rgb, 255:red, 76; green, 97; blue, 226 }  ,opacity=1 ]  {$\theta_{1}$};
\draw (88.91,184.36) node [anchor=north west][inner sep=0.75pt]    {$\textcolor[rgb]{0.96,0.64,0.11}{\theta}\textcolor[rgb]{0.96,0.64,0.11}{_{2}}$};
\draw (31.52,18.6) node [anchor=north west][inner sep=0.75pt]  [font=\footnotesize] [align=left] {parameter space};
\draw (261.55,30.32) node [anchor=north west][inner sep=0.75pt]  [font=\footnotesize] [align=left] {good downstream\\performance};
\draw (185.55,99.32) node [anchor=north west][inner sep=0.75pt]  [font=\footnotesize] [align=left] {\textcolor[rgb]{0.29,0.38,0.89}{useful inductive bias}};
\draw (37.55,209.32) node [anchor=north west][inner sep=0.75pt]  [font=\footnotesize] [align=left] {\textcolor[rgb]{0.96,0.64,0.11}{not useful inductive bias}};

\end{tikzpicture}
        \caption{\textbf{Illustration of how inductive biases can affect identifiability:} In the saturation regime, training can result in different parameters $\theta_1,\theta_2$ with the same training and test loss, but different downstream performance. Even if the loss is insensitive to a model property that is required for good downstream performance, choosing a {\color{blue}useful} inductive bias can help capture said property, overcoming its non-identifiability.
        }
        \label{fig:prior}
    \end{figure}

        \textbf{Compositional generalization} 
        describes a model's capability to comprehend the properties of combined features by understanding multiple individual features: \eg, that by understanding words like ``one," ``cat," and ``black",  the model understands phrases like ``one cat" or ``black cat".  Compositional generalization enables models to perform better on \gls{ood} data. Despite compositionality being a core property of natural languages, theoretical works are only started emerging for \glspl{llm}~\citep{ahuja_provable_2024,han_towards_2024,ramesh_compositional_2024,lake_human-like_2023,nogueira2021investigating,dziri2023faith,saparov2023testing}, in contrast to, \eg, computer vision and object-centric representation learning~\citep{locatello_object-centric_2020,brady_provably_2023,wiedemer_compositional_2023,wiedemer_provable_2023,lowe_rotating_2023, li-etal-2019-compositional}.

        \textbf{Systematic generalization} composes rules
        \cite{bahdanau2019systematic,gao_systhematic,ruis2022improving}, and not syntactical compositions as compositional generalization does. Systematic generalization is relevant in reasoning by ensuring that \glspl{llm} can infer ``Flo is the granddaughter of Betty", from the samples ``Nat is the granddaughter of Betty", ``Greg is the brother of Nat", ``Flo is the sister of Greg"~\cite{systematic_transformer}.
        Despite its relevance, systematic generalization research for \glspl{llm} is still limited \cite{systematic_transformer}.

        \textbf{Symbolic generalization} is the model's ability to transfer the responses learned from the training data to another situation which is symbolically related \cite{hoon2020symbolic}. Relevant examples come from mathematics, \eg, knowing that adding ``$2+5x$" and ``$3+4x$" yields ``$(2+3)+(5+4)x$", 
        a model that generalizes symbolically should correctly infer that adding ``$A+Bx$" and ``$C+Dx$" yields ``$(A+C)+(B+D)x$". Here the ability to move from concrete examples to general examples with symbols is symbolic generalization. The active interest in \glspl{llm}'s use for solving complex mathematical problems~\citep{RomeraParedes2023,liu2023improving,azerbayev2023llemma,imani2023mathprompter} makes symbolic generalization a relevant concept to study.
\vspace{0.5em}

\textbf{Research questions.}

            \begin{researchquestion}
                \citet{von_oswald_transformers_2023} showed that for linear attention only Transformers, in-context learning is implemented by the model in the form of implicit gradient descent updates on the in-context examples in the prompt. This meta learning view of ICL could be leveraged to provide bounds on in-context learning accuracy via meta learning generalization bounds~\citep{PAC-meta}. Could we obtain non-vacuous bounds beyond these simple cases?
         \end{researchquestion}

            \begin{researchquestion}
                 How to formalize the out-of-distribution nature of NLP tasks? For example, what is a useful way to describe the part of the pre-training distribution’s support with vanishingly small probabilities (e.g., ICL prompts)
                Is it possible to connect the algorithmic (e.g., Kolmogorov) complexity to qualitative model properties and better (compositional) generalization?
            \end{researchquestion}

            \begin{researchquestion}
                How can we relate symbolic generalization or reasoning to the probability distributions? How do we accomplish this, for example, for symbolic ODEs?
                Context: \citet{becker_predicting_2023} proposed a Transformer-based method for learning a symbolic representation of ODEs. Since the data-generating process of such ODEs yields a stochastically generated tree, learning a probabilistic model of symbolic ODEs connects symbolic generalization and probabilistic modeling. Indeed, as the authors show in their Fig. 3., the Transformer learns a symbolic ODE that can extrapolate.
            \end{researchquestion}

            \begin{researchquestion}
                Can the findings from the literature that demonstrate the linearity of \glspl{llm} in certain cases~\citep{hernandez_linearity_2024,merullo_language_2024}  be used to explain OOD extrapolation in LLMs?
            \end{researchquestion}

\subsection{Computational language modeling to study transferability}\label{subsec:disc_comp_lang_models}
    In \cref{better-gen-measures}, we advocated for more appropriate generalization measures that reflect the properties of natural languages. For studying phenomena such as compositionality, we can rely on the advancements in computational linguistics and theoretical computer science. We believe that using well-defined formal languages as our testbed and developing computational models for \glspl{llm} is a promising direction to understand (near) \gls{ood} behavior in language models.
    
\vspace{0.5em}
    \textbf{Formal languages.} Natural languages' expressive power hinges on their (approximate) discrete compositionality~\citep{languageinstinct}, \ie, combining a discrete set of symbols according to the language rules. 
        Since formal languages such as \glspl{pcfg} can be constructed to be discrete compositional, they are an excellent testbed to study compositionality in \glspl{llm} (\cf our experiments in \cref{sec:case-study-1}). Namely, we can define the union or intersection of multiple grammars and study extrapolation and transfer performance. Indeed, many researchers rely on \glspl{pcfg} for understanding aspects of natural languages~\citep{liu_same_2023,favre_contextual_2020,merrill_formal_2023,ackerman_survey_2020}. 

\vspace{0.5em}
    \textbf{A computational perspective on language modeling.} Computational models of (formal) languages can provide new insights for understanding \glspl{llm}. Regular languages can be modelled as \glspl{fsm}, and \glspl{rnn} have the same computational model~\citep{cleeremans_finite_1989}; whereas \glspl{pcfg} need an additional stack~\citep{companion2021chomsky}. Recently, \citet{weiss_thinking_2021} showed that a new programming language, \acrshort{rasp}, describes a computational model for Transformers,
        which can explain how reordering fully connected and attention layers changes performance~\citep{press_improving_2020}. In terms of \acrshort{rasp}, these reorderings constrain information flow, acting as an architectural inductive bias.
        As we detail in \cref{subsec:disc_indctive_biases},  computational models of \glspl{llm} and algorithmic information theory can help characterize the inductive biases in \glspl{llm}.

\vspace{0.5em}
        \textbf{Research questions.}

            \begin{researchquestion}
                Is it possible with computational models (such as the RASP language) to define the minimal set of requirements for models (Transformers, state-space models such as Mamba) to perform well on transfer tasks? Are all components in these architectures necessary? Is there a component beyond the architecture (such as the optimization algorithm or weight initialization) without which good transfer is not guaranteed?
            \end{researchquestion}

            \begin{researchquestion}
                Are the models with good fine-tuning performance related in weight space or function space, e.g., is there a (linear) mode connectivity result that describes the geometry of LLMs that transfer well to other tasks?
            \end{researchquestion}

\subsection{Inductive biases for understanding \glspl{llm}} \label{subsec:disc_indctive_biases}

In contrast to relying solely on inductive biases enabling statistical generalization, we advocate for studying inductive biases that are not problem- or loss-specific. These qualitative characteristics remain insightful even in the saturation regime, as they enable us to reason about performance on new tasks and alternative generalization measures (\cref{better-gen-measures}). We encourage investigations that intertwine statistical generalization with these \gls{llm}-relevant inductive biases, \eg, by characterizing extrapolation performance in terms of statistical generalization ability \textit{and} the presence of an inductive bias.
To motivate the need for \gls{llm}-relevant inductive biases, we showcase (sometimes toy) examples of qualitative properties relevant to specific DNN models and tasks. We then outline some promising directions for \glspl{llm}.

\vspace{0.5em}
\textbf{Examples of qualitative model properties.}
\textit{Sparsity} is a prevalent concept in machine learning often enforced through explicit regularisation \citep[see][for a review]{Viadurre2013}. Intriguingly, inductive biases alone can give rise to sparsity in gradient descent: $L-$layer linear diagonal networks trained on binary classification
converge to the $\ell_{\frac2L}$ large margin classifier, yielding a sparse solution
~\citep{gunasekar2019implicit}, whereas deep matrix factorization is known to lead to low-rank solutions \citep{gunasekar2017matrix_fact, arora2019matrix_fact}. For models where the \gls{ntk} assumptions hold, gradient descent solves kernel ridge ($\ell_1$-regularized) regression \citep{jacot_neural_2020}. For DNNs implementing Boolean functions, the resulting parameter to function map is \textit{simple}\footnote{The study of this object is motivated by the observation that SGD approximates Bayesian inference sufficiently well, where the prior $p(f)$ is taken as the probability of a randomly initialized neural network implementing a specific function} in terms of Lempel-Ziv complexity \citep{valle-perez2018deep, dingle_2018} and converges towards low-entropy functions \citep{mingard2020low_entropy}. Binary classifiers of bitstrings are biased towards low sensitivity to changes in the input \citet{depalma2019random}. \citet{rahaman2019spectral} highlights a bias towards low-frequency functions.
There is also work that looks at the dynamics of qualitative properties during training: neural networks appear to learn increasingly complex functions, starting with linear functions \citep{arpit2017krueger, simplicitybiasSGD} making use of higher-order statistics only in later stages \citep{refinetti2023neural}. Although many of these findings rely on simplified mathematical models they nevertheless provide good insights into qualitative properties of trained neural networks, which can be connected to properties of interest such as \gls{ood} extrapolation.%

\vspace{0.5em}
\textbf{Insights from algorithmic information theory.}
The Kolmogorov complexity $K(x)$ of a bitstring $x$ is defined as the length of the shortest program under a fixed programming language that produces $x$ \citep{KOLMOGOROV1998387}.  For \glspl{llm}, an intriguing direction is connecting model properties to the Kolmogorov complexity of its generated text: a bias towards low Kolmogorov complexity might imply improved (compositional) generalization. 
 Though Kolmogorov complexity is uncomputable, insights from algorithmic information theory remain pertinent for understanding LLMs and building general-purpose models \citep{SCHMIDHUBER1997857, hutter2000theory}. 
 \citet{no_free_lunch_agw} argues that real-world data has low complexity in the Kolmogorov sense. This simplicity bias in data is shared with (even randomly initialized) neural networks and is more general than what the architecture would suggest:  CNNs can effectively learn tabular data despite their lack of spatial structure.
 Via the connection between prediction and compression \citep{vitanyi}, we may interpret a \acrshort{dnn} as a compressor of the training data, where the best possible compressor has the lowest Kolmogorov complexity.  \citet{deepmind_llm_compression} shows that successful LLMs are good general-purpose compressors, \eg, Chinchilla 70B compresses ImageNet patches to $43.4\%$ despite having been trained primarily on text. In addition, \citet{deepmind2024universal_predictors} develop a meta-learning method to train \glspl{llm} to approximate Solomonoff Induction \citep{SOLOMONOFF19641}, which also provides interesting connections to algoritmic information theory.
To study the effects of data complexity on \glspl{llm}, we advocate for the use of simple formal languages, such as \glspl{pcfg}~\citep{liu_same_2023,favre_contextual_2020,merrill_formal_2023,ackerman_survey_2020}, as they have a controllable notion of complexity and structure.

\vspace{0.5em}
\textbf{Insights from the Transformer architecture}
The underlying Transformer architecture may shape the inductive biases in LLMs. \acrshort{rasp}~\citep{weiss_thinking_2021} as a computational model for Transformers (\cref{subsec:disc_comp_lang_models}) offers a framework for characterizing the algorithms LLMs can implement. Specifically, the efficiency and compactness with which these algorithms can be expressed in RASP might serve as a novel, \gls{llm}-specific complexity metric. 
    A related direction for finding inductive biases based on the Transformer architecture is mechanistic interpretability~\citep{olah2023mechanistic}, which aims to understand the internal mechanisms of a model.

\vspace{0.5em}
    \textbf{Research questions.}

    \begin{researchquestion}
        How do we go beyond the current study of inductive bias? We believe that fruitful potential directions include studying the computational models and information flow of SOTA models using an established toolbox from other fields such as the theory of finite-state automata, as well as modeling Transformers in the RASP language.
    \end{researchquestion}

    \begin{researchquestion}
        What would happen by iteratively training GPT models to imitate each other? What are the “eigenfeatures” these models amplify, i.e., what are they sensitive to (akin to the periodic features CNNs learn)?
    \end{researchquestion}

    \begin{researchquestion}
        Does model architecture and data contain an implicit bias towards ICL? 
        What would happen if we trained a Transformer adversarially (such that it is not an in-context learner), how many (compared to the size of the original training set) adversarial samples do we need to suppress ICL? If we then train a second Transformer to imitate the first one, does ICL emerge in the second? If we add negligible adversarial perturbation to the data, can ICL still emerge?
    \end{researchquestion}

   \section{Conclusion}
        In this paper, we articulated the position that
        deep learning theory needs yet another shift to understand emergent \gls{llm} behaviours. 
        LLMs often operate in the saturation regime---\ie, they reach (near-)optimal training and test loss---, making statistical generalization insufficient to explain desired model properties because of its ``black-box'' nature and \acrshort{iid} assumption. Therefore, we should focus on inductive biases that can induce qualitative properties relevant for natural language tasks.
        We support our position by demonstrating multiple ways in which AR LMMs are non-identifiable, but due to inductive biases, they can exhibit \gls{ood} rule extrapolation (\cref{sec:case-study-1}), \acrlong{icl} (\cref{sec:case-study-2}), and data-efficient fine-tunability (\cref{sec:case-study-3}).
        Our case studies formulate well-defined scenarios that can serve as a basis to study \glspl{llm} theoretically.  We encourage the community to (1) study more expressive generalization measures; (2) use formal languages and computational models to study \gls{llm} behaviour in well-controlled settings; and (3) analyze and propose helpful inductive biases that ensure good \gls{ood} performance and transferability.

\section*{Acknowledgements}
 The authors would like to thank the anonymous reviewers for their comments improving the manuscript and inspiring some of the research questions, Siyuan Guo for fruitful comments on OOV generalization, Bence Nyéki for insights on practical aspects of natural language processing, Felix Leeb for discussions on LLM reasoning, Vince Velkey for constructive comments on the manuscript, and Gail Weiss for her insights on PCFGs. This work was supported by a Turing AI World-Leading Researcher Fellowship G111021. Patrik Reizinger acknowledges his membership in the European Laboratory for Learning and Intelligent Systems (ELLIS) PhD program and thanks the International Max Planck Research School for Intelligent Systems (IMPRS-IS) for its support. This work was supported by the German Federal Ministry of Education and Research (BMBF): Tübingen AI Center, FKZ: 01IS18039A. Wieland Brendel acknowledges financial support via an Emmy Noether Grant funded by the German Research Foundation (DFG) under grant no. BR 6382/1-1 and via the Open Philantropy Foundation funded by the Good Ventures Foundation. Wieland Brendel is a member of the Machine Learning Cluster of Excellence, EXC number 2064/1 – Project number 390727645. This research utilized compute resources at the Tübingen Machine Learning Cloud, DFG FKZ INST 37/1057-1 FUGG. Anna Kerekes would like to thank the Max Planck ETH Center for Learning Systems (CLS).

\section*{Impact Statement}
    A deeper understanding of contributing factors to the success LLMs has potentially positive impacts: it may lead to increased data-, cost- and energy-efficiency, and thus broader access to the benefits, and it may lead to better guarantees and predictions of model behaviour that increase safety and interpretability. Our work advocates for a more application-driven theory of deep learning (whereby researchers take into account ways in which these models are used) but falls short of making recommendations that these application-driven theories should incorporate sociotechnical aspects or account for models’ effect on humans.

\bibliography{references,imported_references}
\bibliographystyle{icml2024}

\newpage
\appendix
\onecolumn
\section{Details on $\varepsilon-$identifiability and the saturation regime}
\label{appx:bicikli}
Mathematical formalizations of in-context learning often assume perfect (statistical) generalization \citep{in_context_bayesian, wang2023large}, that is, $\kl{p(x_{1:k})}{q(x_{1:k})}=0$. In this context, the saturation regime is understood as the regime of perfect generalization. However, in experimental demonstrations, the term ``saturation regime’’ is used more leniently to mean \textit{near} perfect test loss. We argue that a refinement of the concept is required in order to align theory with practice. As we demonstrate in \cref{sec:case-study-2}, by relaxing perfect generalization only to $\kl{p(x_{1:k})}{q(x_{1:k})}\leq\varepsilon$, we may observe qualitatively different behaviours in models, for example, the existence and non-existence of in-context learning ability. Hence for the theory, it does matter whether we are truly or only approximately in the saturation regime. Yet in practice, in-context learning properties hold even for smaller transformers, where the test loss is only near-optimal at best (Figure 6 in \citet{liu_same_2023}). We argue that this discrepancy between theory and practice may be explained by inductive biases: in the near-optimal loss regime, where multiple models of varying quality exist, inductive biases select a solution that satisfies additional important properties, such as in-context learning.

\section{Problem framework and notations}
Our framework and notations are based on those in \citep{in_context_bayesian}, but we provide some additional details.
\label{appx:framework}
\makenomenclature

\nomenclature[01]{\(\theta \in \Theta\)}{space of latent concepts, the concept $\theta$ determines the transition probability matrix of the \gls{hmm} hidden states}
\nomenclature[02]{\(p_0\)}{a prior distribution on $\Theta$}
\nomenclature[03]{\(o \in \mathcal{O}\)}{a single token and an output state of the HMM. And $\mathcal{O}$ is the set of all tokens}
\nomenclature[04]{\(h \in \mathcal{H}\)}{a hidden state of the HMM}
\nomenclature[05]{\( o^{\text{delim}} \in \mathcal{O} \)}{a unique delimiter token and an output state of the HMM} 

\nomenclature[06]{\(h^{\text{delim}} \in \mathcal{D} \subset \mathcal{H} \)}{a hidden delimiter state of the HMM}
\nomenclature[07]{\(T\)}{the (fixed) length of each pre-training document $(o_1, o_2, ..., o_T)$ }

\nomenclature[08]{\((x_i, y_i) \in \mathcal{O} ^{k} \)}{an in-context learning example} 
\nomenclature[09]{\(k\)}{the length of each in-context learning example, i.e. $x_i \in \mathcal{O} ^{k-1}$, $y_i \in  \mathcal{O}  $}
\nomenclature[10]{\(S_n\)}{the concatenation of $n$ in-context learning examples separated by delimiter tokens, $S_n=\{(x_i, y_i, o^{\text{delim}}) \vert 1 \leq i \leq n \}$}
\nomenclature[11]{\( x_{\text{test}}\)}{a test example for in-context learning, with length $k-1$}
\nomenclature[12]{\( (S_n, x_{\text{test}}) \in  \mathcal{P}_N\)}{the full in-context learning prompt with $n$ examples and a test example}
\nomenclature[13]{\(n\)}{the number of in-context learning examples in $S_n$}
\nomenclature[14]{\(N\)}{the (fixed) length of the in-context learning prompts, $N=(k+1)(n+1)-1$}
\nomenclature[15]{\(\mathcal{P}_N\ \subset \mathcal{O}^N \)}{the set of sequences of length $N$ having prompt structure (Def \ref{prompt_str_def})}
\nomenclature[16]{\(y\)}{the in-context learning output}
\nomenclature[17]{\(p_{\text{prompt}}\)}{the distribution used to generate in-context prompts $(S_n, x_{\text{test}})$, defined on $\mathcal{O}^N \times  \mathcal{H} \times \mathcal{D}^n$ (see below for a full definition). Note that $p_{\text{prompt}}$ different from $p$ due to the distribution shift induced by the concatenation with the delimiter tokens.}
\nomenclature[18]{\(p(\cdot \vert \theta) \)}{the distribution of token sequences (of some fixed maximal length) given the latent concept, defined by the HMM}
\nomenclature[19]{\(q_n\)}{a distribution on token sequences, depends on $n$, the number of examples in the in-context learning prompt}

\printnomenclature

Let $p_0(\theta)$ be a prior distribution on the latent concepts $\theta \in \Theta$. For each $\theta$, let the distribution $p(o_1, \dots, o_T | \theta)$ be a \gls{hmm}, where the concept $\theta$ determines
the transition probability matrix of the \gls{hmm}. From  this \gls{hmm}, we generate token sequences $o_1, \dots, o_T
 $ of a fixed pre-training document length $T$. First, a concept $\theta$ is sampled from $p_0(\theta)$, then a document is generated according to $p(o_1, \dots, o_T | \theta)$.  
 Therefore, the pre-training distribution (on token sequences) is a mixture of \glspl{hmm} induced by these conditionals.
    \begin{equation}
        p (o_1, \dots, o_T) = \int_{\theta \in \Theta} p(o_1, \dots, o_T | \theta) p_0(\theta) \, d \theta.
    \end{equation}
    Now we detail how the in-context prompts are generated. This iterative process defines $p_{\text{prompt}}$. The in-context prompt consists of example input-output pairs $(x_i, y_i)$ and a test input $x_{\text {test }}$, and the goal is to predict the test output $y_{\text{test}}$ by predicting the next token. All examples in the \gls{icl} prompt are connected via the same underlying concept $\theta^*$---\eg, in the input-output pairs ("Gandhi was", "Indian"), ("Jefferson was", "American"), the underlying concept is the nationality of a person.
    The $i$-th example pair is independently generated as follows:
    \begin{enumerate}[nolistsep]
        \item Generate a start hidden state $h_{i}^{\text {start }}$ from a prompt start distribution $p_{\text {prompt}}(h)$. \textit{This defines $p_{\text {prompt}}$ on $\mathcal{H}$ as an arbitrary discrete distribution.}
        \item Given $h_{i}^{\text {start }},$ generate the example sequence $(x_{i}, y_{i})$ from $p\left(o_1,o_2, \dots, o_k  \mid h_{i}^{\text {start }}, \theta^{*}\right)$, the pretraining distribution conditioned on a prompt concept $\theta^{*}$  (given by a HMM). \textit{This defines $p_{\text {prompt}}(o_1,o_2, \dots,  o_k  \mid h_{i}^{\text {start}}):=p(o_1,o_2, ... o_k  \mid h_{i}^{\text {start}}, \theta^*).$}
    \item The delimiter state at the end of each example (except the test example) is sampled from  $p_{\text {prompt}}(h^{\text{delim}})$. \textit{This defines $p_{\text{prompt}}$ on $\mathcal{D}$ as an arbitrary discrete distribution.}   \end{enumerate}
    The test input $x_{\text {test}}=x_{n+1}$ is sampled in the same way, i.e. $x_{\text {test}} \sim p\left(o_1,o_2, ... o_{k-1} \mid h_{n+1}^{\text {start }}, \theta^{*}\right)$.  The prompt consists of a sequence of in-context examples $S_{n}$ followed by the test example $x_{\text {test }}$, with a unique delimiter token $o^{\text {delim }}$ between each element. 
    \begin{flalign}
        ( S_n, x_{\text{test}}) &= (x_1, y_1, o^{\text{delim}}, x_2, y_2, o^{\text{delim}},
        \dots, x_n, y_n, o^{\text{delim}}, x_{\text{test}}) 
    \end{flalign}
The above determines the probability of $( S_n, x_{\text{test}})$ under $p_{\text{prompt}}$ as
\begin{equation}
\begin{split}
    &p_{\text{prompt}}(S_n, x_{\text{test}})=\int_{\mathcal{H}^{n+1}}  \left(\prod_{j=1}^n p_{\text{prompt}}(h_j^{\text{start}})p(o_{1:k}^j \mid h_j^{\text{start}}, \theta^*) p_{\text{prompt}}(h_j^{\text{delim}})  \right) \cdot \\
    &\left(p_{\text{prompt}}(h_{n+1}^{\text{start}})p(o_{1:k-1}^j \mid h_{n+1}^{\text{start}}, \theta^*) \right) \mathrm{d} h_1^{\text{start}}\mathrm{d} h_2^{\text{start}} \dots \mathrm{d} h_{n+1}^{\text{start}}.
\end{split}
\end{equation}
Note that it is not necessary to integrate with respect to the $h_j^{\text{delim}}$ due to Assumption \ref{assumption1} below.
After sampling $(S_n, x_{\text{test}}) \sim p_{\text{prompt}},$ we treat them as fixed values without loss of generality.

For in-context learning, the ground-truth output $y_{\text{test}}$ for the example $x_{\text{test}}$ is defined by $y_{\text{test}}=\argmax_y p_{\text{prompt}} (y|x_{\text{test}})$, where $p_{\text{prompt}} (y|x_{\text{test}})$ can be calculated as
\begin{equation}
\begin{split}
    &p_{\text{prompt}}(y|x_{\text{test}}) = \int_{\mathcal{H}} p_{\text{prompt}}(y, h_{\text{test}}^{\text{start}}|x_{\text{test}}) \,d h_{\text{test}}^{\text{start}} = \int_{\mathcal{H}} p_{\text{prompt}}( h_{\text{test}}^{\text{start}}|x_{\text{test}}) p_{\text{prompt}}(y | h_{\text{test}}^{\text{start}},x_{\text{test}}) \,d h_{\text{test}}^{\text{start}} =\\
    &=\int_{\mathcal{H}} p_{\text{prompt}}( h_{\text{test}}^{\text{start}}|x_{\text{test}}) p(y | h_{\text{test}}^{\text{start}},x_{\text{test}}, \theta^*) \,d h_{\text{test}}^{\text{start}} = \mathbb{E}_{h^{\text{start}}_{\text{test}} \sim p_{\text{prompt}}(h^{\text{start}}_{\text{test}}|x_{\text{test}})} [p(y|h^{\text{start}}_{\text{test}}, x_{\text{test}}, \theta^* )],
\end{split}
\end{equation}
where $h^{\text{start}}_{\text{test}}$ is the hidden state corresponding to the first token of $x_{\text{test}}$.\\
Then we focus on the in-context learning predictor $\argmax_y p(y|S_n, x_{\text{test}})$, which predicts the output with the highest probability over the pre-training distribution given the prompt. We say that the model is an in-context learner
    if, as $n\! \rightarrow\! \infty,$ 
    \begin{equation}
        \argmax_y p(y|S_n, x_{\text{test}}) \rightarrow y_{\text{test}}.
    \end{equation}

\begin{assumption}[Delimiter hidden states]
\label{assumption1}
Let the delimiter hidden states $\mathcal{D}$ be a subset of $\mathcal{H}$. For any $h^{\text {delim }} \in \mathcal{D}$ and $\theta \in \Theta, p\left(o^{\text {delim }} \mid h^{\text {delim }}, \theta\right)=1$ and for any $h \notin \mathcal{D}, p\left(o^{\text {delim }} \mid h, \theta\right)=0$.
\end{assumption}

\begin{assumption}[Bound on delimiter transitions]
\label{assumption2}
For any delimiter state $h^{\text {delim }} \in \mathcal{D}$ and any hidden state $h \in \mathcal{H}$, the probability of transitioning to a delimiter hidden state under $\theta$ is upper bounded $p\left(h^{\text {delim }} \mid h, \theta\right)<$ $c_{2}<1$ for any $\theta \in \Theta \backslash\left\{\theta^{*}\right\}$, and is lower bounded $p\left(h^{\text {delim }} \mid h, \theta^{*}\right)>c_{1}>0$ for $\theta^{*}$. Additionally, the start hidden state distribution for delimiter hidden states is bounded as $p\left(h^{\text {delim }} \mid \theta\right) \in\left[c_{3}, c_{4}\right]$.
\end{assumption}

The above two assumptions allow us to simplify our analysis and avoid degenerate cases such as a deterministic (hidden) Markov chain. The result of \citep{in_context_bayesian} (which we rely on) used 3 additional assumptions (Assumption 3,4,5), but those are omitted here.

\section{Proof of \cref{our_theorem}}
\label{appx:proof}
\ourtheorem*

We denote the threshold example sequence length as $n_0$, after which $p$ satisfies in-context learning, i.e., $$\forall n>n_0 : \argmax_y p(y|S_n, x_{\text{test}}) = \argmax_y p_{\text{prompt}} (y|x_{\text{test}}).$$ 

\begin{proof} 

Our proof follows the below steps.

\begin{itemize}[nolistsep]
            \item \textbf{Step 1}:  for every $n\geq n_0$, we define a $q_n$ by equating it with $p$ everywhere except on sequences that have a prompt structure (to be defined). We construct $q_n$ such that the prompt completion will be different than in $p$, i.e. $$\argmax_y q_n(y|S_n, x_{\text{test}}) \neq \argmax_y p(y|S_n, x_{\text{test}}).$$ We do this by making sure that $$\argmax_{y \neq y^*} q_n(y|S_n, x_{\text{test}}) \geq q_n(y^*|S_n, x_{\text{test}})+ \frac{\delta}{2}.$$
            \item \textbf{Step 2}: we bound $\text{KL} \left( p||q_n\right)$ as $$\text{KL} \left( p||q_n\right) \leq \text{[constant]} \times \text{[the probability of prompts]}.$$
            \item \textbf{Step 3}: we show that the latter converges to $0$ as $n \to \infty$ and is controlled by a function of $\delta$.
     \end{itemize}

\paragraph{Step 1}
Let us denote the length $N$ prompt by $O=(o_1, ..., o_N)$.
Consider the fixed distribution $p(o_1, \dots, o_N)$ defined by a mixture of HMMs. For any fixed $n \in \mathbb{Z}^{+}$, we define a distribution $q_n(o_1, \dots, o_N)$ as a modification of $p$. To define $q_n$, we make use of the following definition.

\begin{definition}(Prompt structure)
\label{prompt_str_def}
    We say that a sequence of tokens $(o_1, ..., o_N)$ has prompt structure if it can be written in form $(S_n, x_{\text{test}}, y)=(x_1, y_1, o^{\text{delim}}, x_2, y_2, o^{\text{delim}},\dots, x_n, y_n, o^{\text{delim}}, x_{\text{test}}, y)$, with each $x_i$ having length $k-1$ and $y_i$ having length 1.
\end{definition}
Let us denote the set of prompt structures in $\mathcal{O}^N$ as $\mathcal{P}_N.$ We consider those sequences that have a prompt structure. We construct $q_n$ such that it is different only on these sequences and equal to $p(o_1, \dots, o_N)$ everywhere else.

We can expand $q_n$ on prompt structures via the chain rule
\begin{flalign}
    q_n(S_n, x_{\text{test}},y) &= \sum_{j=1}^{n+1} q_n(y_j | S_{j-1},x_j) q_n(x_j|S_{j-1}) q_n(o^{\text{delim}}|S_{j-2}, x_{j-1}, y_{j-1})  \label{chain_rule}\\
    &=q_n(y | S_{n},x_{\text{test}}) q_n(x_{\text{test}}|S_{n}) q_n(o^{\text{delim}}|S_{n-1}, x_{n}, y_{n})\nonumber\\
    &+ \sum_{j=1}^{n} q_n(y_j | S_{j-1},x_j) q_n(x_j|S_{j-1}) q_n(o^{\text{delim}}|S_{j-2}, x_{j-1}, y_{j-1}),
\end{flalign}
with notation $x_{n+1}=x_{\text{test}}$, $y_{n+1}=y$ and $x_0=y_0=S_0=S_{-1}=\emptyset$ in order to define the endpoints appropriately.
For $j=1, \dots, n$ let $$q_n(x_j|S_{j-1}) := p(x_j|S_{j-1}) $$ and $$q_n(o^{\text{delim}}|S_{j-2}, x_{j-1}, y_{j-1}) := p(o^{\text{delim}}|S_{j-2}, x_{j-1}, y_{j-1}),$$ 
i.e., the same as $p$. We are only modifying $p(y | S_{n},x_{\text{test}})$, and only at its largest and second largest values---as we show, that is sufficient to change the output of the $\argmax$, thus, changing whether the model has \gls{icl}. Denote
\begin{align*}
    a_1 &= \max_y p(y | S_{n}, x_{\text{test}})\\
    y_1^* &= \argmax_y p(y | S_{n}, x_{\text{test}})\\
    a_2 &= \max_{y \neq y_1^*} p(y | S_{n}, x_{\text{test}})\\
    y_2^* &= \argmax_{y \neq y_1^*} p(y | S_{n}, x_{\text{test}}).
\end{align*}
Then, for all $y \neq y_1^*$ and $y\neq y_2^*$ let $q_n$ be the same as $p,$ \ie, 
\begin{align*}
    q_n(y| S_{n}, x_{\text{test}}) := p(y | S_{n}, x_{\text{test}}) \quad \forall y_1^* \neq y \neq y_2^*.
\end{align*}

However, on the largest and second largest values we change $p$ as the following
\begin{align*}
    q_n(y^*_1| S_{n}, x_{\text{test}}) := \frac{a_1+a_2}{2}- \frac{\delta}{2}\\
    q_n(y^*_2| S_{n}, x_{\text{test}}) := \frac{a_1+a_2}{2}+ \frac{\delta}{2},
\end{align*}
where $\delta$ is arbitrarily small. Due to \eqref{chain_rule}, $q_n$ is well-defined. Since $q_n(y| S_{n}, x_{\text{test}})$ has its maximum at $y_2^*$,
$$ \argmax_y p(y| S_{n}, x_{\text{test}})) \neq \argmax_y q_n(y| S_{n}, x_{\text{test}}),$$
that is, the model $q_n$ has a different output when the $\argmax$ operator is applied.
 Since $n>n_0$,
from \citet{in_context_bayesian} we know that $p$ has \gls{icl}: 
$$
\argmax_y p(y|S_n, x_{\text{test}}) = \argmax_y p_{\text{prompt}} (y|x_{\text{test}}) .$$
However, since $p$ and $q_n$ do not match, the following holds and $q_n$ \textbf{cannot be an in-context learner}:
$$
\argmax_y q_n(y|S_n, x_{\text{test}}) \neq \argmax_y p_{\text{prompt}} (y|x_{\text{test}}) $$

To bound $\text{KL} \left( p||q_n\right)$ in Step 2, we need the following inequality

\begin{equation}
\label{bound}
    \log \left(  \frac{p(y|S_n, x_{\text{test}})}{q_n(y|S_n, x_{\text{test}})} \right) \leq \log \left( \frac{2}{1- \delta}\right). 
\end{equation}
It holds since if $q_n(y|S_n, x_{\text{test}}) = p(y|S_n, x_{\text{test}})$, then the log equals to 0, otherwise by \eqref{bound}:
$$
\log \left(  \frac{p(y_1^*|S_n, x_{\text{test}})}{q_n(y_1^*|S_n, x_{\text{test}})} \right)  = \log \left( \frac{2a_1}{a_1+a_2 - \delta} \right) \leq \log \left( \frac{2}{1- \delta}\right) \quad \text{and}
$$
$$
\log \left(  \frac{p(y_2^*|S_n, x_{\text{test}})}{q_n(y_2^*|S_n, x_{\text{test}})} \right) = \log \left( \frac{2a_2}{a_1+a_2 + \delta} \right) \leq \log \left( \frac{2}{1- \delta}\right),
$$
since $\frac{2a_2}{a_1+a_2+\delta} < \frac{2a_1}{a_1+a_2-\delta} \leq \frac{2}{1-\delta}$ with equality if $a_2=0$. 
\paragraph{Step 2} Now we bound the KL divergence between $p(o_1, \dots, o_N)$ and $q_n(o_1, \dots, o_N)$.
\begin{equation*}
\begin{split}
    &\text{KL} ( p(o_1, \dots, o_N) || q_n(o_1, \dots, o_N)) = \sum_{t \in \mathcal{O}^N} p(t) \log \left( \frac{p(t)}{q_n(t)}) \right) = \sum_{t \in \mathcal{P}_N}p(t) \log \left( \frac{p(t)}{q_n(t)} \right) +\\
    &+\sum_{ t  \notin \mathcal{P}_N} p(t) \log \left( \frac{p(t)}{q_n(t)} \right) = \sum_{(S_n, x_{\text{test}}, y)} p(S_n, x_{\text{test}}, y) \log \left( \frac{p(S_n, x_{\text{test}}, y)}{q_n(S_n, x_{\text{test}}, y)} \right) =\\
\end{split}
\end{equation*}
where the last equation is due to the fact that $q_n$ is defined to equal $p$ on non-prompt structures ($\mathcal{O}^N \setminus \mathcal{P}_N$). Expanding $p(S_n, x_{\text{test}}, y)$ and $q_n(S_n, x_{\text{test}}, y)$ via the chain rule, we get 
\begin{equation*}
\begin{split}
    &= \sum_{(S_n, x_{\text{test}}, y)} p(S_n, x_{\text{test}}, y) \log \left( \prod_{j=1}^{n+1} \frac{p(y_j | S_{j-1},x_j) p(x_j|S_{j-1}) p(o^{\text{delim}}|S_{j-2}, x_{j-1}, y_{j-1})}{q_n(y_j | S_{j-1},x_j) q_n(x_j|S_{j-1}) q_n(o^{\text{delim}}|S_{j-2}, x_{j-1}, y_{j-1})} \right) =\\
    &= \sum_{(S_n, x_{\text{test}}, y)} p(S_n, x_{\text{test}}, y)  \log \left(  \frac{p(y|S_n, x_{\text{test}}) }{q_n(y|S_n, x_{\text{test}})} \right),\\  
\end{split}
\end{equation*}
since by the definition of $q_n$, all the terms inside the $\log$ vanish excluding $p(y|S_n, x_{\text{test}})$ and $q_n(y|S_n, x_{\text{test}})$. Now we use the bound in Eq \eqref{bound}.
\begin{equation*}
\begin{split}
    & \text{KL} ( p(o_1, \dots, o_N) || q_n(o_1, \dots, o_N)) \leq  \log \left( \frac{2}{1- \delta}\right) \sum_{(S_n, x_{\text{test}}, y)} p(S_n, x_{\text{test}}, y).
\end{split}
\end{equation*}

\paragraph{Step 3} We now show that $\sum_{(S_n, x_{\text{test}}, y)} p(S_n, x_{\text{test}}, y) \to 0$ as $n \to \infty$ exponentially fast.
\begin{flalign}
    \sum_{(S_n, x_{\text{test}}, y)} p(S_n, x_{\text{test}}, y) &= \sum_{(S_n, x_{\text{test}}, y)} \int_{\theta \in \Theta} p(S_n, x_{\text{test}}, y | \theta) p_0(\theta)\\
    &= \int_{\theta \in \Theta} \sum_{(S_n, x_{\text{test}}, y)}  p(S_n, x_{\text{test}}, y | \theta) p_0(\theta).
\end{flalign}
Let us fix $\theta \in \Theta$. %
Consider the HMM with output distribution $p$, conditioned on $\theta$. We wish to focus on the property of prompt structures that delimiter output states occur at every $(k+1)^{\text{th}}$ token. Bounding the probability of delimiter output states occurring at every $(k+1)^{\text{th}}$ token allows us to bound the probability of prompt structures.

From Assumption \ref{assumption1} of \citet{in_context_bayesian}, we know that a delimiter hidden state generates the delimiter output state with probability 1, and non-delimiter hidden states do not generate the delimiter output state. Hence it is equivalent to bound the probability of the hidden Markov chain being at delimiter hidden states exactly at every $(k+1)^{\text{th}}$ step.

Let us consider those latent concepts $\theta$, where it is possible to reach a delimiter hidden state.
Without loss of generality, assume that our HMM hidden state Markov chain is at a delimiter hidden state (a state in $\mathcal{D}$), otherwise, we reach $\mathcal{D}$ in some steps. For two, possibly equal $h^{\text{delim}, i}, h^{\text{delim}, j} \in \mathcal{D}$, define
\begin{equation}
   p_{ij}^{k, \theta}=\sum_{h_1, ..., h_{k} \in \mathcal{H}} p(h^{\text{delim}, i}, h_{k}, h_{k-1}, ..., h_{1}, h^{\text{delim}, j} | \theta)
\end{equation}
as the probability of the hidden Markov chain going from state $h^{\text{delim}, j}$ to $h^{\text{delim}, i}$ in $k+1$ steps (it may pass through delimiter hidden states in the meantime). For those latent concepts $\theta$, where it is not possible to reach a delimiter hidden state, it is possible to define $p_{ij}^{k, \theta}$ accordingly, but its value is zero. As this does not affect our upper bound, we may neglect these cases. 
Let $p^*=\sup_{\theta \in \Theta \setminus \{ \theta^*\}}\max_{i, j}p_{ij}^{k, \theta}$, where the maximum is under all pairs of delimiter hidden states in $\mathcal{D}$. We show that $p^*<1$. For all $\theta \in \Theta \setminus \theta^*$, by Assumption \ref{assumption2} of \citet{in_context_bayesian}, for all $h^{\text{delim}} \in \mathcal{D}$ and $h \in \mathcal{H}$, $p(h^{\text{delim}} | h, \theta) < c_2 <1$. Hence for all $i, j$, 

\begin{flalign}
    p_{ij}^{k, \theta} &= \sum_{h_1, ..., h_{k} \in \mathcal{H}} p(h^{\text{delim}, i}, h_{k}, h_{k-1}, ..., h_{1}, h^{\text{delim}, j} | \theta)\\
    &= \sum_{h_1, ..., h_{k} \in \mathcal{H}} p(h^{\text{delim}, i}| h_{k}, \theta) p( h_{k}| h_{k-1}, ..., h_{1}, h^{\text{delim}, j},  \theta)p(  h_{k-1}, ..., h_{1}, h^{\text{delim}, j} |  \theta)\\
    &< c_2 \sum_{h_1, ..., h_{k} \in \mathcal{H}} p( h_{k}| h_{k-1}, ..., h_{1}, h^{\text{delim}, j},  \theta)p(  h_{k-1}, ..., h_{1}, h^{\text{delim}, j} |  \theta)\\
    &\leq c_2.
\end{flalign}

Thus, $p^* \leq c_2$, and the probability of generating a prompt structure, which has $n$ instances of this pattern of a delimiter hidden state appearing after $k$ non-delimiter hidden states, is upper bounded by $c_2^n$. 

Hence 
\begin{equation}
\begin{split}
    &\int_{\theta \in \Theta} \sum_{(S_n, x_{\text{test}}, y)}  p(S_n, x_{\text{test}}, y | \theta) p_0(\theta) \leq
     \int_{\theta \in \Theta} (\max_{i, j}p_{ij}^{k, \theta})^n p_0(\theta) =\\ 
     &=\int_{ \theta \in \{ \theta^* \} } (\max_{i, j}p_{ij}^{k, \theta})^n p_0(\theta) + \int_{\theta \in \Theta \setminus \{ \theta^* \} } (\max_{i, j}p_{ij}^{k, \theta})^n p_0(\theta)
      \leq (\sup_{\theta \in \Theta \setminus \{ \theta^* \}}\max_{i, j}p_{ij}^{k, \theta})^n = (p^*)^n \leq c_2^n,
\end{split}
\end{equation}

where the second inequality is because the integral over $\theta \in \{ \theta^* \}$ is zero, since $\{ \theta^* \}$ has 0 Lebesgue measure. Therefore $\sum_{(S_n, x_{\text{test}}, y)} p(S_n, x_{\text{test}}, y)$ decays exponentially.
Hence 
$$
\text{KL} ( p(o_1, \dots, o_N) || q_n(o_1, \dots, o_N))\leq \log \left( \frac{2}{1-\delta}\right) c_2^n $$
From this, we obtain that defining $n_1 = \log_{c_2} \left( \frac{\epsilon}{\log 2} \right) + 1$ and $\delta = \min \left\{a_1 + a_2, 1-2 e^{-\frac{\epsilon}{c_2^n}} \right\}$ ensures
$\text{KL} ( p(o_1, \dots, o_N) || q_n(o_1, \dots, o_N))\leq \epsilon$ for all $n \geq n_1$. 
\end{proof}

\section{Identifiability}\label{sec:app_ident}

    \begin{definition}[Set of probability measures]
        We denote the set of probability measures on domain \domain as $\pdfset(\mathcal{X})$.
    \end{definition}

    \begin{definition}[Property]
        Let $\pdfset(\mathcal{X})$ be the set of distibutions on \domain and let $p\in\pdfset(\mathcal{X})$. A property $\property$ is a binary function $ \pdfset(\mathcal{X})\to\braces{0;1}$.  We say that $p$ has property $\property$, if $\property(p)=1$ and that it does not if $\property(p)=0$ 
    \end{definition}

    \begin{definition}[Property equivalence classes]\label{def:eq_class}
        A property \property partitions  a set of distributions $\pdfset(\mathcal{X})$ into two equivalence classes, $\pdfset_{\property}$ and $\pdfset_{\overline{\property}}$ such that
        \begin{align}
            \forall i\neq j, p_i, p_j \in \pdfset_{\property}: \property(p_i)=\property(p_j)= 1\\
            \forall i\neq j, p_i, p_j \in \pdfset_{\overline{\property}}: \property(p_i)=\property(p_j)= 0
        \end{align}
        such that $\pdfset(\mathcal{X}) = \pdfset_{\property}\cup \pdfset_{\overline{\property}}$ and $\pdfset_{\property}\cap \pdfset_{\overline{\property}}=\emptyset.$
    \end{definition}

    \begin{definition}[$\varepsilon-$non-identifiability of distributional properties]\label{definition:eps_non_ident}
        Let $\pdfset(\domain)$ be a set of distributions with an equivalence class structure, given by property \property and denoted as $\pdfset_{\property}, \pdfset_{\overline{\property}}$. We say that property $\property$ of a distribution  is $\varepsilon-$non-identifiable if there exists a  distribution $p\in \pdfset_{\property}$ such that $\exists q \in \pdfset_{\overline{\property}}$ such that $\kl{p}{q} \leq \varepsilon.$
    \end{definition}
\section{Experimental details}\label{sec:app_exp}

\paragraph{Reproducibility and codebase.} We use PyTorch~\citep{paszke2019pytorch}, PyTorch Lightning~\citep{Falcon_PyTorch_Lightning_2019}, and HuggingFace Transformers~\citep{Wolf_Transformers_State-of-the-Art_Natural_2020}. Our code and experimental logs are publicly available at \texttt{\url{https://github.com/rpatrik96/llm-non-identifiability}}.

\paragraph{\gls{pcfg}.}
    We generate data from the $a^nb^n$ \glspl{pcfg} up to length 256. Besides the tokens $a\ (0)$ and $b\ (1)$, we use \acrshort{sos} (2), \acrshort{eos} (3), and padding (4) tokens. We define our test prompts as all possible sequences of length 8 (prepended with \acrshort{sos}), which we split into in-distribution, and \gls{ood} test prompts, based on whether they can be completed in the form of $a^nb^n$. The training set includes all unique sequences up to length $256$.

\begin{table}[H]
    \caption{\gls{pcfg} parameters}
    \label{tab:pcfg}
    \vskip 0.15in
    \begin{center}
    \begin{small}
    \begin{sc}
    \begin{tabular}{lr} \toprule
        Parameter & Values \\ \midrule
        Number of tokens & 5 (\acrshort{sos}, \acrshort{eos}, PAD, 0,1)\\
        Maximum sequence length & 256\\
        Training data maximum length & 256\\
        Test prompt length & 8\\
        Batch size & $128$\\
    \end{tabular}
    \end{sc}
    \end{small}
    \end{center}
    \vskip -0.1in
\end{table}

\paragraph{Model.} We use a Transformer decoder~\citep{vaswani_attention_2017_edit} in flavor of the decoder-only GPT models~\citep{radford2018improving,radford_language_2018,openai2023gpt4}. We apply standard positional encoding, layer normalization,  ReLU activations, the AdamW optimizer~\citep{loshchilov_decoupled_2019} with inverse square root learning rate schedule~\citep{xiong_layer_2020}.  For prompt prediction, the model can predict up to length $300.$ We train for $50,000$ epochs with the standard \gls{ce} loss for the next token prediction task. For the adversarial and oracle training versions, we add an additional loss term which we detail below.

\begin{table}[H]
    \caption{Transformer parameters}
    \label{tab:transformer}
    \vskip 0.15in
    \begin{center}
    \begin{small}
    \begin{sc}
    \begin{tabular}{lr} \toprule
        Parameter & Value (normal)  \\ \midrule
        Model & Transformer decoder\\
        Number of layers & $5$\\
        Dropout probability & $0.1$\\
        Model dimension & $10$\\
        Feedforward dimension & $1024$\\
        Number of attention heads & $5$\\
        Layer norm $\epsilon$ & \expnum{6}{-3}\\
        Activation &  ReLU\\
        Optimizer & AdamW\\
        Learning rate scheduler & inverse square root\\
        Batch size & $128$\\
        Learning rate & \expnum{2}{-3}\\
        Prompt prediction cutoff length & $300$\\
        Number of epochs & $50,000$\\
    \end{tabular}
    \end{sc}
    \end{small}
    \end{center}
    \vskip -0.1in
\end{table}

\paragraph{Metrics.}
    We monitor training and validation loss, and the adherence to the grammar's two rules \ref{rule1},\ref{rule2}. We measure the accuracy of each separately and simultaneously (\ie, to check whether the generated sequence is grammatical). For a deeper understanding, we calculate these metrics for different scenarios:
    \begin{enumerate}[nolistsep]
        \item For the in- and out-of-distribution test prompts and
        \item For a batch of \acrshort{sos} tokens.
    \end{enumerate}
    For each of the above, we re-calculate the accuracies for the subset of prompt completions which have an \acrshort{eos} token to avoid false conclusions (\eg, if the model wants to finish $aaa$ as a longer sequence than the cutoff length, the unfinished sequence would lower the accuracy). Since for the \gls{ood} prompts, it is by definition impossible to fulfil \ref{rule2} (that $a's$ are before $b's$), we separately calculate this rule on the completion: \eg, if the \gls{ood} $\textbf{abbb}$ is completed as $\textbf{abbb}aa$, then it is considered correct for this metric, but $\textbf{abbb}abaa$ is not, as it has an $a$ after a $b$ in the \textit{ completion}.
    We also monitor the accuracy of next token prediction via greedy decoding (\ie, using the token with the largest probability). We report additional numerical values in \cref{tab:extrapolation}, supplementing \cref{figure:rule_extrapolation}.

\paragraph{Adversarial training.}
    For adversarial training, we generate \gls{ood} sequences such that the number of $a's$ and $b's$ is not equal, there is one more from one symbol. Then, we treat the first 8 $a$ and $b$ tokens (\ie, the same as the test prompt length) as the \textit{prompt}, and the rest as the \textit{completion}. During training, we add a \acrshort{ce} loss on the \gls{ood} prompt completions. The rationale of only optimizing on the \gls{ood} completions is to keep the prompts \gls{ood}, since our claim in \cref{sec:case-study-1} is about different behavior for \gls{ood} prompts.

\paragraph{Oracle training: enforcing rule extrapolation.}
    This scenario is very similar to adversarial training, with the difference, that we generate additional \gls{ood} training samples, where the \textit{prompt} is still \gls{ood}, but here the \textit{completion} is generated such that the number of $a's$ and $b's$ is equal over the whole sequence. Then we add a \acrshort{ce} loss on the \gls{ood} prompt completions.

\begin{table}[H]
    \caption{Comparison of the extrapolation performance of \acrshort{mle}, adversarial, and oracle training for \gls{ood} prompts. For (approximately) the same validation loss, the extrapolation of \ref{rule1} for \gls{ood} prompts differs enormously, showing that the loss alone cannot distinguish the extrapolation property}
    \label{tab:extrapolation}
    \vskip 0.15in
    \begin{center}
    \begin{small}
    \begin{sc}
    \begin{tabular}{lrrr} \toprule
        \multirow{2}{*}{Name} &\multirow{2}{*}{Validation loss} & \multicolumn{2}{c}{Accuracy of \ref{rule1}}\\
        & & Mean+std.& Range\\ \midrule
        \acrshort{mle} & $0.0215\scriptscriptstyle\pm0.0011$& $0.437{\scriptscriptstyle\pm0.047}$ &\brackets{0.339;0.629}\\
        Adversarial & $0.0223\scriptscriptstyle\pm0.00094$& $0.$ & \brackets{0.;0.}\\
        Oracle &$0.0199\scriptscriptstyle\pm0.00025$ &$0.83{\scriptscriptstyle\pm0.122}$ &  \brackets{0.634; 1.} \\
        \bottomrule
    \end{tabular}
    \end{sc}
    \end{small}
    \end{center}
    \vskip -0.1in
\end{table}

\newpage
\printacronyms

\end{document}